%% file: main.tex
\documentclass[11pt]{article}
\usepackage{enumerate}
\usepackage[OT1]{fontenc}
\usepackage{amsmath,amssymb}
\usepackage[usenames]{color}
\usepackage[font=small]{caption}
\usepackage{subcaption}

\usepackage{smile}
\usepackage{multirow}
\usepackage[colorlinks,
            linkcolor=red,
            anchorcolor=blue,
            citecolor=blue
            ]{hyperref}

\def\given{\,|\,}
\def\biggiven{\,\big{|}\,}
\def\tr{\mathop{\text{tr}}\kern.2ex}

\newcommand{\JJ}{\mathbb{J}}

\long\def\comment#1{}

\def\tr{\mathop{\text{Tr}}}

\def\cS{{\mathcal{S}}}

\newcommand{\bel}{\begin{eqnarray}\label}
\newcommand{\eel}{\end{eqnarray}}
\newcommand{\bes}{\begin{eqnarray*}}
\newcommand{\ees}{\end{eqnarray*}}

\usepackage{mathrsfs}

\usepackage{fullpage}

\def \doo {{\rm do}}

\def\##1\#{\begin{align}#1\end{align}}
\def\$#1\${\begin{align*}#1\end{align*}}
\usepackage{enumitem}

\begin{document}

\title{Provably Efficient Causal Reinforcement Learning with Confounded Observational Data}
\date{\today}
\author
{
\normalsize Lingxiao Wang\thanks{Northwestern University; \texttt{lwang@u.northwestern.edu}}
\qquad
\normalsize Zhuoran Yang\thanks{Princeton University; \texttt{zy6@princeton.edu}}
\qquad
\normalsize Zhaoran Wang\thanks{Northwestern University; \texttt{zhaoranwang@gmail.com}}
}

\maketitle

\input{abs}
\input{intro}

\section{Confounded Reinforcement Learning}
\vskip4pt
\noindent{\bf Structural Causal Model.}
We denote a structural causal model (SCM) \citep{pearl2009causality} by a tuple $(A, B, F, P)$. Here $A$ is the set of exogenous (unobserved) variables, $B$ is the set of endogenous (observed) variables,  $F$ is the set of structural functions capturing the causal relations, which determines an endogenous variable $v \in B$ based on the other exogenous and endogenous variables, and $P$ is the distribution of all the exogenous variables. We say that a pair of variables $Y$ and $Z$ are confounded by a variable $W$ if they are both caused by $W$. 

An intervention on a set of endogenous variables $X \subseteq B$ assigns a value $x$ to $X$ regardless of the other exogenous and endogenous variables as well as the structural functions. We denote by $\doo(X = x)$ the intervention on $X$ and write $\doo(x)$ if it is clear from the context.  Similarly, a stochastic intervention \citep{munoz2012population, diaz2019causal} on a set of endogenous variables $X \subseteq B$  assigns a distribution $p$ to $X$ regardless of the other exogenous and endogenous variables as well as the structural functions. We denote by $\doo(X\sim p)$ the stochastic intervention on $X$.

\vskip4pt
\noindent{\bf Confounded Markov Decision Process.}
\label{sec::MDP_env}
To characterize a Markov decision process (MDP) in the offline setting with observational data, which are possibly confounded, we introduce an SCM, where the endogenous variables are the states $\{s_h\}_{h\in[H]}$, actions $\{a_h\}_{h\in[H]}$, and rewards $\{r_h\}_{h\in[H]}$. 
Let $\{w_h\}_{h\in[H]}$ be the confounders. In \S\ref{sec::backdoor_LSVI}, we assume that the confounders are partially observed, while in \S\ref{sec::frontdoor_LSVI}, we assume that they are unobserved, both in the offline setting. The set of structural functions $F$ consists of the transition of states $s_{h+1}\sim\cP_h(\cdot\given s_h, a_h, w_h)$, the transition of confounders $w_{h}\sim \tilde\cP_h(\cdot\given s_h)$, the behavior policy $a_h\sim\nu_h(\cdot\given s_h, w_h)$, which depends on the confounder $w_h$, and the reward function $r_h(s_h, a_h, w_h)$. See Figure \ref{fig::backdoor_1} for the causal diagram that describes such an SCM. 
\begin{figure*}[ht!]
    \centering
    \begin{subfigure}[t]{0.5\textwidth}
        \centering
        \includegraphics[height=1.2in]{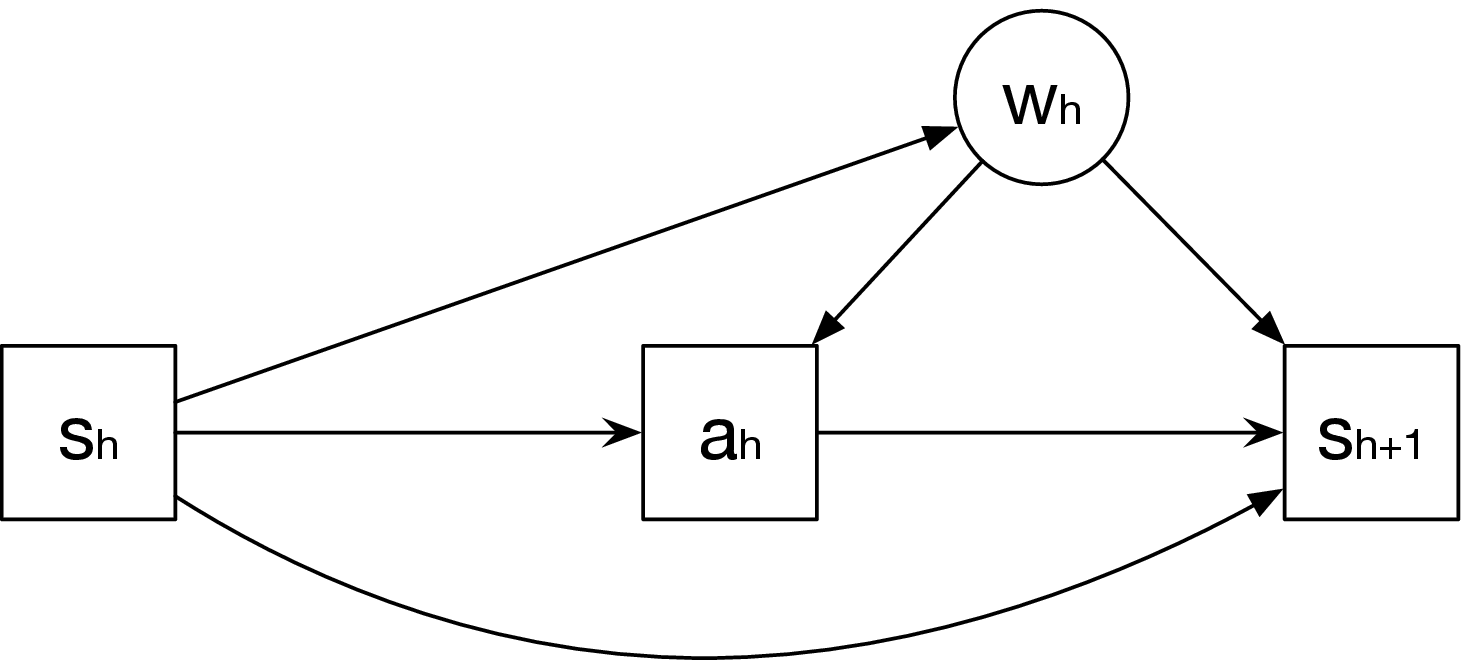}
        \caption{Offline Setting}
    \end{subfigure}%
    ~ 
    \begin{subfigure}[t]{0.5\textwidth}
        \centering
        \includegraphics[height=1.2in]{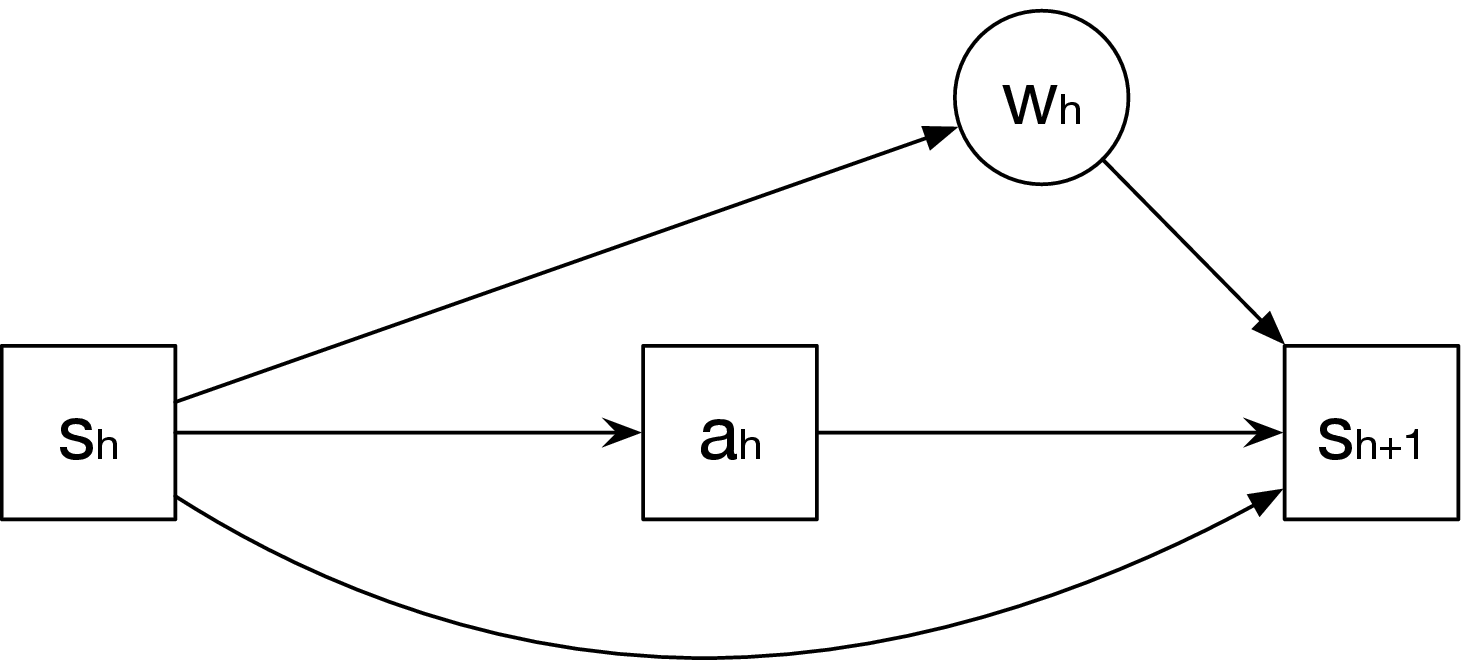}
        \caption{Online Setting}
    \end{subfigure}
    \caption{Causal diagrams of the $h$-th step of the confounded MDP (a) in the offline setting and (b) in the online setting, respectively.}
    \label{fig::backdoor_1}
\end{figure*}

Here $a_h$ and $s_{h+1}$ are confounded by $w_h$ in addition to $s_h$. We denote such a confounded MDP by the tuple $(\cS, \cA, \cW, H, \overline\cP, r)$, where $H$ is the length of an episode, $\cS$, $\cA$, and $\cW$ are the spaces of states, actions, and confounders, respectively, $r = \{r_h\}_{h\in[H]}$ is the set of reward functions, and $\overline\cP = \{\cP_h, \tilde \cP_h\}_{h \in H}$ is the set of transition kernels. In the sequel, we assume without loss of generality that $r_h$ takes value in $[0, 1]$ for all $h\in[H]$.

In the online setting that allows for intervention, we assume that the confounders $\{w_h\}_{h\in[H]}$ are unobserved. A policy $\pi = \{\pi_h\}_{h\in[H]}$ induces the stochastic intervention $\doo(a_1 \sim \pi_1(\cdot\given s_1), \ldots, a_H \sim \pi_H(\cdot\given s_H))$, which does not depend on the confounders. In particular, an agent interacts with the environment as follows. At the beginning of the $k$-th episode, the environment arbitrarily selects an initial state $s^k_1$ and the agent selects a policy $\pi^k = \{\pi^k_h\}_{h\in[H]}$. At the $h$-th step of the $k$-th episode, the agent observes the state $s^k_h$ and takes the action $a^k_h\sim \pi^k_{h}(\cdot\given s^k_h)$. The environment randomly selects the confounder $w^k_{h}\sim \tilde \cP_h(\cdot\given s^k_h)$, which is unobserved, and the agent receives the reward $r^k_h = r_h(s^k_h, a^k_h, w^k_h)$. The environment then transits into the next state $s^k_{h+1} \sim \cP_h(\cdot\given s^k_h, a^k_h, w^k_h)$.

For a policy $\pi = \{\pi_h\}_{h\in H}$, which does not depend on the confounders $\{w_h\}_{h\in[H]}$, we define the value function $V^{\pi} = \{V^{\pi}_h\}_{h\in[H]}$ as follows,
\#\label{eq::def_V}
V^{\pi}_{h}(s)&=\EE\biggl[\sum^H_{j = h} r_j(s_j, a_j, w_j) ~\bigg |~ s_h = s,  s_{j+1}\sim \cP_h(\cdot\given s_j, a_j, w_j), w_j\sim\tilde\cP_h(\cdot\given s_j), \doo\bigl(a_j\sim \pi_j(\cdot\given s_j)\bigr) \biggr] \notag\\
&= \EE_{\pi}\biggl[\sum^H_{j = h} r_j(s_j, a_j, w_j) ~\bigg |~ s_h = s \biggr], \quad \forall h \in[H],
\#
where we denote by $\EE_\pi$ the expectation with respect to the confounders $\{w_j\}_{j=h}^H$ and the trajectory $\{(s_j, a_j)\}_{j=h}^H$, starting from the state $s_j = s$ and following the policy $\pi$.~Correspondingly, we define the action-value function $Q^{\pi} = \{Q^{\pi}_h\}_{h\in[H]}$ as follows,
\#\label{eq::def_Q}
Q^{\pi}_{h}(s, a) = \EE_{\pi}\biggl[\sum^H_{j = h} r_j(s_j, a_j, w_j) ~\bigg |~ s_h = s, \doo(a_h = a) \biggr], \quad \forall h \in[H].
\#
We assess the performance of an algorithm using the regret against the globally optimal policy $\pi^* = \{\pi_h^*\}_{h\in[H]}$ in hindsight after $K$ episodes, which is defined as follows,
\#\label{eq::def_regret}
\textrm{Regret}(T) = \max_{\pi}\sum^K_{k = 1}\bigl(V^{\pi}_1(s^k_1) - V^{\pi^k}_1(s^k_1)\bigr)= \sum^K_{k = 1}\bigl(V^{\pi^*}_1(s^k_1) - V^{\pi^k}_1(s^k_1)\bigr).
\#
Here $T = HK$ is the total number of steps.


Our goal is to design an algorithm that minimizes the regret defined in \eqref{eq::def_regret}, where $\pi^*$ does not depend on the confounders $\{w_h\}_{h\in[H]}$. In the online setting that allows for intervention, it is well understood how to minimize such a regret \citep{jaksch2010near, azar2017minimax, jin2018q, jin2019provably}. However, it remains unclear how to efficiently utilize the observational data obtained in the offline setting, which are possibly confounded. In real-world applications, e.g., autonomous driving and personalized medicine, such observational data are often abundant, whereas intervention in the online setting is often restricted.
\vskip4pt
\noindent{\bf Why is Incorporating Confounded Observational Data Challenging?}
Straightforwardly incorporating the confounded observational data into an online algorithm possibly leads to an undesirable regret due to the mismatch between the online and offline data generating processes. 
In particular, due to the existence of the confounders $\{w_h\}_{h\in[H]}$, which are partially observed (\S\ref{sec::backdoor_LSVI}) or unobserved (\S\ref{sec::frontdoor_LSVI}), the conditional probability $\PP(s_{h+1}\given s_h, a_h)$ in the offline setting is different from the causal effect $\PP(s_{h+1}\given s_h, \doo(a_h))$ in the online setting \citep{peters2017elements}. More specifically, it holds that
\$
&\PP(s_{h+1}\given s_h, a_h) = \frac{\EE_{w_h\sim\tilde\cP_h(\cdot\given s_h)}\bigl[\cP_h(s_{h+1}\given s_h, a_h, w_h)\cdot\nu_h(a_h\given s_h, w_h) \bigr] }{\EE_{w_h\sim\tilde\cP_h(\cdot\given s_h)}\bigl[\nu_h(a_h\given s_h, w_h) \bigr] },\notag\\
&\PP\bigl(s_{h+1}\biggiven s_h, \doo(a_h)\bigr) = \EE_{w_{h}\sim \tilde\cP_h(\cdot\given s_h)}\bigl[\cP_h(\cdot\given s_h, a_h, w_h)\bigr].
\$
In other words, without proper covariate adjustments \citep{pearl2009causality}, the confounded observational data may be not informative for estimating the transition dynamics and the associated action-value function in the online setting. To this end, we propose an algorithm that incorporates the confounded observational data in a provably efficient manner. Moreover, our analysis quantifies the amount of information carried over by the confounded observational data from the offline setting and to what extent it helps reducing the regret in the online setting.

In what follows, we discuss the connection between confounded MDP and other extensions of MDP and SCM.
\begin{itemize}
\item{\bf Dynamic Treatment Regimes (DTR).}
In a DTR \citep{zhang2019near}, all the states $\{s_h\}_{h\in[H]}$ are confounded by a common confounder $w$, whereas in a confounded MDP, each state $s_h$ depends on an individual confounder $w_{h-1}$, which further depends on the previous state $s_{h-1}$. If $w_{h-1}$ does not depend on $s_{h-1}$, the confounded MDP reduces to a DTR by summarizing the confounders into $w = (w_1, \ldots, w_H)$.
\item{\bf Contextual MDP (CMDP).}
A confounded MDP is similar to a CMDP \citep{hallak2015contextual} if we cast the confounders $\{w_h\}_{h\in[H]}$ as the context therein. In a CMDP, which focuses on the online setting, the context is fixed throughout an episode, whereas in a confounded MDP, the confounders $\{w_h\}_{h\in[H]}$ vary across the $H$ steps. Moreover, in a CMDP, the goal is to minimize the regret against the globally optimal policy that depends on the context, which is a stronger benchmark than $\pi^*$ in \eqref{eq::def_regret}, since $\pi^*$ does not depend on the confounders $\{w_h\}_{h\in[H]}$.

\item{\bf Partially Observable MDP (POMDP).}
A confounded MDP is a simplified POMDP \citep{tennenholtz2019off} if we cast the confounders $\{w_h\}_{h\in[H]}$ as the hidden states therein (assuming that the confounders are unobserved in the offline setting as in \S\ref{sec::frontdoor_LSVI}). 
A POMDP is more challenging to solve, since marginalizing over the hidden states does not yield an MDP, which is the case in a confounded MDP.
\end{itemize}

\section{Algorithm and Theory for Partially Observed Confounder}
\label{sec::backdoor_LSVI}
In this section, we propose the Deconfounded Optimistic Value Iteration (DOVI) algorithm. DOVI handles the case where the confounders are unobserved in the online setting but are partially observed in the offline setting. We then characterize the regret of DOVI.
\subsection{Algorithm}
\vskip4pt
\noindent{\bf Backdoor Adjustment.}
In the online setting that allows for intervention, the causal effect of $a_h$ on $s_{h+1}$ given $s_h$, that is, $\PP(s_{h+1}\given s_h, \doo(a_h))$, plays a key role in the estimation of the action-value function. Meanwhile, the confounded observational data may not allow us to identify the causal effect $\PP(s_{h+1}\given s_h, \doo(a_h))$ if the confounder $w_h$ is unobserved. However, if the confounder $w_h$ is partially observed in the offline setting, the observed subset $u_h$ of $w_h$ allows us to identify the causal effect $\PP(s_{h+1}\given s_h, \doo(a_h))$, as long as $u_h$ satisfies the following backdoor criterion.
\begin{assumption}[Backdoor Criterion \citep{pearl2009causality, peters2017elements}]
\label{asu::backdoor_crit}
In the SCM defined in \S\ref{sec::MDP_env} and its induced directed acyclic graph (DAG), for all $h\in[H]$, there exists an observed subset $u_h$ of $w_h$ that satisfies the backdoor criterion, that is,
\begin{itemize}
\item the elements of $u_h$ are not the descendants of $a_h$, and
\item conditioning on $s_h$, the elements of $u_h$ $d$-separate every path between $a_h$ and $s_{h+1}$ that has an incoming arrow into $a_h$.
\end{itemize}
\end{assumption}
See Figure \ref{fig::backdoor_2} for an example that satisfies the backdoor criterion. In particular, we identify the causal effect $\PP(s_{h+1}\given s_h, \doo(a_h))$ as follows.
\begin{proposition}[Backdoor Adjustment \citep{pearl2009causality}]
\label{lem::backdoor}
Under Assumption \ref{asu::backdoor_crit}, it holds for all $h\in[H]$ that
\$
\PP\bigl(s_{h+1} ~\big|~ s_h, \doo(a_h)\bigr) &= \EE_{u_h\sim \PP(\cdot\given s_h)} \bigl[\PP(s_{h+1} ~|~ s_h, a_h, u_h)\bigr],\\
\EE\bigl[r_h(s_h, a_h, w_h)\biggiven s_h, \doo(a_h)\bigr] &= \EE_{u_h\sim \PP(\cdot\given s_h)}\Bigl[\EE\bigl[r_h(s_h, a_h, w_h)\biggiven s_h, a_h, u_h\bigr]\Bigr].
\$
Here $(s_{h+1}, s_h, a_h, u_h)$ follows the SCM defined in \S\ref{sec::MDP_env}, which generates the confounded observational data.
\end{proposition}
\begin{proof}
See \cite{pearl2009causality} for a detailed proof.
\end{proof}
\begin{figure*}[ht!]
    \centering
    \includegraphics[height=1.2in]{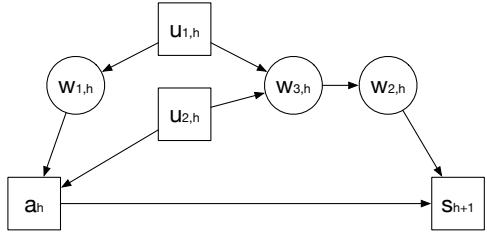}
    \caption{An illustration of the backdoor criterion. The causal diagram corresponds to the $h$-th step of the confounded MDP conditioning on $s_h$. Here $w_h = \{w_{1, h},w_{2, h},w_{3, h},u_{1, h},u_{2, h}\}$ is the confounder and the subset $u_h = \{u_{1, h}, u_{2,h}\}$ satisfies the backdoor criterion.}
    \label{fig::backdoor_2}
\end{figure*}
With a slight abuse of notation, we write $\PP(s_{h+1} ~|~ s_h, a_h, u_h)$ as $\cP_h(s_{h+1}\given s_h, a_h, u_h)$ and $\PP(u_h\given s_h)$ as $\tilde\cP_h(u_h\given s_h)$, since they are induced by the SCM defined in \S\ref{sec::MDP_env}. In the sequel, we define $\cU$ the space of observed state $u_h$ and write $r_h = r_h(s_h, a_h, w_h)$ for notational simplicity.

\vskip4pt
\noindent{\bf Backdoor-Adjusted Bellman Equation.}
We now formulate the Bellman equation for the confounded MDP. It holds for all $(s_h, a_h)\in\cS\times\cA$ that
\$
Q^{\pi}_{h}(s_h, a_h) &= \EE_{\pi}\biggl[\sum^H_{j = h} r_j(s_j, a_j, u_j) ~\bigg |~ s_h, \doo(a_h) \biggr]= \EE\bigl[r_h\biggiven s_h, \doo(a_h)\bigr] + \EE_{s_{h+1}}\bigl[V^{\pi}_{h+1}(s_{h+1})\bigr],
\$
where $\EE_{s_{h+1}}$ denotes the expectation with respect to $s_{h+1} \sim \PP(\cdot\biggiven s_h, \doo(a_h))$. Here $\EE[r_h\biggiven s_h, \doo(a_h)]$ and $\PP(\cdot\biggiven s_h, \doo(a_h))$ are characterized in Proposition \ref{lem::backdoor}. In the sequel, we define the following transition operator and counterfactual reward function,
\#
\label{eq::def_operator_backdoor}(\PP_h V)(s_h, a_h) &= \EE_{s_{h+1}\sim\PP(\cdot\given s_h, \doo(a_h))}\bigl[V(s_{h+1})\bigr], \quad \forall V:\cS\mapsto\RR, ~(s_h, a_h) \in \cS\times\cA,\\
\label{eq::def_R_backdoor}R_h(s_h, a_h) &= \EE\bigl[r_h\biggiven s_h, \doo(a_h)\bigr], \quad \forall (s_h, a_h) \in \cS\times\cA.
\#
We have the following Bellman equation,
\#\label{eq::bellman_backdoor}
Q^{\pi}_{h}(s_h, a_h) = R_h(s_h, a_h) + (\PP_hV^{\pi}_{h+1})(s_h, a_h),\quad \forall h\in[H], ~(s_h, a_h) \in \cS\times\cA.
\#
Correspondingly, the Bellman optimality equation takes the following form,
\#\label{eq::bellman_backdoor_opt}
Q^{*}_{h}(s_h, a_h) = R_h(s_h, a_h) + (\PP_hV^{*}_{h+1})(s_h, a_h), \quad V^*_h(s_h) = \max_{a_h\in\cA} Q_h^*(s_h, a_h),
\#
which holds for all $h\in[H]$ and $(s_h, a_h) \in \cS\times\cA$. Such a Bellman optimality equation allows us to adapt the least-squares value iteration (LSVI) algorithm \citep{bradtke1996linear, jaksch2010near,osband2014generalization, azar2017minimax, jin2019provably}.

\vskip4pt
\noindent{\bf Linear Function Approximation.}
We focus on the following setting with linear transition kernels and reward functions \citep{yang2019sample, yang2019reinforcement, jin2019provably, cai2019provably}, which corresponds to a linear SCM \citep{peters2017elements}.
\begin{assumption}[Linear Confounded MDP]
\label{asu::linMDP_backdoor}
We assume that
\$
\cP_h(s_{h+1} \given s_h, a_h , u_h) = \langle \phi_h(s_h, a_h, u_h), \mu_h(s_{h+1}) \rangle,\quad \forall  h\in[H],~(s_{h+1}, s_h, a_h)\in\cS\times\cS\times\cA,
\$
where $\phi_h(\cdot, \cdot, \cdot)$ and $\mu_h(\cdot) = (\mu_{1, h}(\cdot), \ldots, \mu_{d, h}(\cdot))^\top$ are $\RR^d$-valued functions. We assume that $\sum^d_{i = 1}\|\mu_{i, h}\|^2_1 \leq d$ and $\|\phi_h(s_h, a_h, u_h)\|_2 \leq 1$ for all $h\in[H]$ and $(s_h, a_h, u_h)\in\cS\times\cA\times\cU$. Meanwhile, we assume that
\#
\EE[r_h\given s_h, a_h, u_h] = \phi_h(s_h, a_h, u_h)^\top \theta_h, \quad \forall  h\in[H],~(s_h, a_h, u_h)\in\cS\times\cA\times\cU,
\#
where $\theta_h \in \RR^d$ and $\|\theta_h\|_2\leq \sqrt{d}$ for all $h\in[H]$.
\end{assumption}

Such a linear setting generalizes the tabular setting where $\cS$, $\cA$, and $\cU$ are finite. 
\begin{proposition}
\label{prop::backdoor_feature}
We define the backdoor-adjusted feature as follows,
\#\label{eq::def_back_adjust_feature}
\psi_h(s_h, a_h) = \EE_{u_h\sim\tilde \cP_h(\cdot\given s_h)}\bigl[\phi_h(s_h, a_h, u_h)\bigr], \quad \forall  h\in[H],~(s_h, a_h)\in\cS\times\cA.
\#
Under Assumption \ref{asu::backdoor_crit}, it holds that 
\$
\PP(s_{h+1}\given s_h, \doo(a_h)) = \langle\psi_h(s_h, a_h), \mu_h(s_{h+1})\rangle, \quad \forall  h\in[H],~(s_{h+1}, s_h, a_h)\in\cS\times\cS\times\cA.
\$
Moreover, the action-value functions $Q^\pi_h$ and $Q^*_h$ are linear in the backdoor-adjusted feature $\psi_h$ for all $\pi$. 
\end{proposition}
\begin{proof}
See \S\ref{pf::backdoor_feature} for a detailed proof.
\end{proof}

Such an observation allows us to estimate the action-value function based on the backdoor-adjusted features $\{\psi_h\}_{h\in[H]}$ in the online setting. See \S\ref{sec::mechanism} for a detailed discussion. In the sequel, we assume that either the density of $\{\tilde\cP_h(\cdot\given s_h)\}_{h\in[H]}$ is known or the backdoor-adjusted feature $\{\psi_h\}_{h\in[H]}$ is know.

In the sequel, we introduce the DOVI algorithm (Algorithm \ref{alg::backdoor}). Each iteration of DOVI consists of two components, namely point estimation, where we estimate $Q^*_h$ based on the confounded observational data and the interventional data, and uncertainty quantification, where we construct the upper confidence bound (UCB) of the point estimator.

\begin{algorithm}[htpb]
\caption{Deconfounded Optimistic Value Iteration (DOVI) for Confounded MDP}
\begin{algorithmic}[1]\label{alg::backdoor}
\REQUIRE Observational data $\{(s_{h}^i, a_{h}^i, u_h^i, r_h^i)\}_{i \in [n], h \in [H]}$, tuning parameters $\lambda, \beta >0$, backdoor-adjusted feature $\{\psi_h\}_{h\in[H]}$, which is defined in \eqref{eq::def_back_adjust_feature}.
\STATE{\bf Initialization:} Set $\{Q^0_h, V^0_h\}_{h\in[H]}$ as zero functions and $V^k_{H+1}$ as a zero function for $k\in[K]$. 
\FOR{$k = 1, \ldots, K$}
\FOR{$h = H, \ldots, 1$}
\STATE \label{line::mechanism_backdoor} Set $\omega^k_h \leftarrow  \argmin_{\omega\in\RR^d} \sum^{k-1}_{\tau = 1}(r_h^\tau + V^\tau_{h+1}(s^\tau_{h+1})-\omega^\top \psi_h(s^\tau_h, a^\tau_h))^2 + \lambda \|\omega\|^2_2 + L^k_h(\omega)$, where $L^k_h$ is defined in \eqref{eq::reg_backdoor}.
\STATE \label{line::UCB} Set $Q^k_h(\cdot, \cdot) \leftarrow  \min\{\psi_h(\cdot, \cdot)^\top \omega^k_h + \Gamma^k_h(\cdot, \cdot), H-h\}$, where $\Gamma^k_h$ is defined in \eqref{eq::def_Gamma_backdoor}.
\STATE Set $\pi^k_h(\cdot\given s_h) \leftarrow \argmax_{a_h\in\cA} Q^k_h(s_h, a_h)$ for all $s_h\in\cS$.
\STATE \label{line::V} Set $V^k_h(\cdot) \leftarrow \langle \pi^k_h(\cdot\given\cdot),  Q^k_h(\cdot, \cdot)\rangle_\cA$.
\ENDFOR
\STATE Obtain $s^k_1$ from the environment.
\FOR{$h = 1, \ldots, H$}
\STATE Take $a^k_h\sim \pi^k_h(\cdot\given s^k_h)$. Obtain $r^k_h = r_h(s^k_h, a^k_h, u^k_h)$ and $s^k_{h+1}$.
\ENDFOR
\ENDFOR
\end{algorithmic}
\end{algorithm}

\vskip4pt
\noindent{\bf Point Estimation.}
To solve the Bellman optimality equation in \eqref{eq::bellman_backdoor_opt}, we minimize the empirical mean-squared Bellman error as follows at each step,
\#\label{eq::LSVI_backdoor}
\omega^k_h \leftarrow \argmin_{\omega\in\RR^d} \sum^{k-1}_{\tau = 1}\bigl(r_h^\tau + V^\tau_{h+1}(s^\tau_{h+1})-\omega^\top \psi_h(s^\tau_h, a^\tau_h)\bigr)^2 + \lambda \|\omega\|^2_2 + L^k_h(\omega), ~~h = H, \ldots, 1,
\#
where we set $V^k_{H+1} = 0$ for all $k\in[K]$ and $V^\tau_{h+1}$ is defined in Line \ref{line::V} of Algorithm \ref{alg::backdoor} for all $(\tau, h)\in[K]\times[H-1]$. Here $k$ is the index of episode, $\lambda >0$ is a tuning parameter, and  $L^k_h$ is a regularizer, which is constructed based on the confounded observational data. More specifically, we define
\#\label{eq::reg_backdoor}
L^k_h(\omega) = \sum^{n}_{i = 1}\bigl( r_h^i + V^k_{h+1}(s^i_{h+1})-\omega^\top \phi_h(s^i_h, a^i_h, u^i_h)\bigr)^2, \quad \forall (k, h) \in[K]\times[H],
\#
which corresponds to the least-squares loss for regressing $r^i_h + V^k_{h+1}(s^i_{h+1})$ against $\phi_h(s^i_h, a^i_h, u^i_h)$ for all $i\in[n]$. Here $\{(s^i_h, a^i_h, u^i_h, r^i_h)\}_{(i, h)\in[n]\times[H]}$ are the confounded observational data, where $u^i_h\sim \tilde\cP_h(\cdot\given s^i_h)$, $s^i_{h+1}\sim\cP_h(\cdot\given s^i_h, a^i_h, u^i_h)$, and $a^i_h\sim\nu_h(\cdot\given s^i_h, w^i_h)$ with $\nu = \{\nu_h\}_{h\in[H]}$ being the behavior policy. Here recall that, with a slight abuse of notation, we write $\PP(s_{h+1} ~|~ s_h, a_h, u_h)$ as $\cP_h(s_{h+1}\given s_h, a_h, u_h)$ and $\PP(u_h\given s_h)$ as $\tilde\cP_h(u_h\given s_h)$, since they are induced by the SCM defined in \S\ref{sec::MDP_env}.

The update in \eqref{eq::LSVI_backdoor} takes the following explicit form,
\#\label{eq::backdoor_update}
\omega^k_h \leftarrow (\Lambda^{k}_h)^{-1} \biggl( &\sum^{k-1}_{\tau = 1} \psi_h(s^\tau_h, a^\tau_h) \cdot \bigl(V^k_{h+1}(s^\tau_{h+1}) + r_h^\tau\bigr)\notag\\
&\qquad+ \sum^n_{i = 1} \phi_h(s_h^i, a_h^i, u_h^i) \cdot \bigl(V^k_{h+1}(s^i_{h+1})+ r_h^i\bigr) \biggr),
\#
where 
\#\label{eq::backdoor_var}
\Lambda^k_h = \sum^{k-1}_{\tau = 1} \psi_h(s^\tau_h, a^\tau_h) \psi_h(s^\tau_h, a^\tau_h)^\top+ \sum^n_{i = 1} \phi_h(s_h^i, a_h^i, u_h^i)\phi_h(s_h^i, a_h^i, u_h^i)^\top + \lambda I.
\#
\vskip4pt
\noindent{\bf Uncertainty Quantification.}~We now construct the UCB $\Gamma_h^k(\cdot, \cdot)$ of the point estimator $\psi_h(\cdot, \cdot)^\top \omega^k_h$ obtained from \eqref{eq::backdoor_update}, which encourages the exploration of the less visited state-action pairs. To this end, we employ the following notion of information gain to motivate the UCB,
\#\label{eq::backdoor_ucb_1}
\Gamma^k_h(s^k_h, a^k_h) \propto  H(\omega^k_h\given \xi_{k-1}) - H\bigl(\omega^k_h\given \xi_{k-1}\cup\{(s^k_h, a^k_h)\}\bigr),
\#
where $H(\omega^k_h\given \xi_{k-1})$ is the differential entropy of the random variable $\omega^k_h$ given the data $\xi_{k-1}$. In particular, $\xi_{k-1} = \{(s^\tau_h, a^\tau_h, r^\tau_h)\}_{(\tau, h)\in[k-1]\times[H]}\cup\{(s^i_h, a^i_h, u^i_h, r^i_h)\}_{(i, h)\in[n]\times[H]}$ consists of the confounded observational data and the interventional data up to the $(k-1)$-th episode. However, it is challenging to characterize the distribution of $\omega^k_h$. To this end, we consider a Bayesian counterpart of the confounded MDP, where the prior of $\omega^k_h$ is $N(0, \lambda I)$ and the residual of the regression problem  in \eqref{eq::LSVI_backdoor} is $N(0, 1)$. In such a ``parallel" confounded MDP, the posterior of $\omega^k_h$ follows $N(\mu_{k, h}, (\Lambda^k_h)^{-1})$, where $\Lambda^k_h$ is defined in \eqref{eq::backdoor_var} and $\mu_{k, h}$ coincides with the right-hand side of \eqref{eq::backdoor_update}. Moreover, it holds for all $(s^k_h, a^k_h)\in\cS\times\cA$ that
\$
&H(\omega^k_h \given \xi_{k-1}) = 1/2\cdot \log\det\bigl((2\pi e)^d\cdot (\Lambda^k_h)^{-1} \bigr),\notag\\
&H\bigl(\omega^k_h\biggiven \xi_{k-1}\cup\{(s^k_h, a^k_h)\}\bigr) = 1/2\cdot \log\det\Bigl((2\pi e)^d\cdot \bigl(\Lambda^k_h + \psi_h(s^k_h, a^k_h)\psi_h(s^k_h, a^k_h)^\top\bigr)^{-1} \Bigr). \
\$
Correspondingly, we employ the following UCB, which instantiates \eqref{eq::backdoor_ucb_1}, that is,
\#\label{eq::def_Gamma_backdoor}
\Gamma^k_h(s^k_h, a^k_h) = \beta\cdot\Bigl(\log \det\bigl(\Lambda^k_h + \psi_h(s^k_h, a^k_h)\psi_h(s^k_h, a^k_h)^\top\bigr) - \log\det(\Lambda^k_h)\Bigr)^{1/2}
\#
for all $(s^k_h, a^k_h) \in \cS\times\cA$. Here $\beta > 0$ is a tuning parameter. We highlight that, although the information gain in \eqref{eq::backdoor_ucb_1} relies on the ``parallel" confounded MDP, the UCB in \eqref{eq::def_Gamma_backdoor}, which is used in Line \ref{line::UCB} of Algorithm \ref{alg::backdoor}, does not rely on the Bayesian perspective. Also, our analysis establishes the frequentist regret.

\vskip4pt
\noindent{\bf Regularization with Observational Data: A Bayesian Perspective.}
In the ``parallel" confounded MDP, it holds that
\$
\omega^k_h\sim N(0, \lambda I), \quad \omega^k_h\given\xi_{0} \sim N\bigl(\mu_{1, h}, (\Lambda^1_h)^{-1}\bigr), \quad \omega^k_h\given\xi_{k-1} \sim N\bigl(\mu_{k, h}, (\Lambda^{k}_h)^{-1}\bigr),
\$
where $\mu_{k, h}$ coincides with the right-hand side of \eqref{eq::backdoor_update} and $\mu_{1, h}$ is defined by setting $k = 1$ in $\mu_{k, h}$. Here $\xi_0 = \{(s^i_h, a^i_h, u^i_h, r^i_h)\}_{(i, h)\in[n]\times[H]}$ are the confounded observational data.
Hence, the regularizer $L^k_h$ in \eqref{eq::reg_backdoor} corresponds to using $\omega^k_h\given\xi_{0}$ as the prior for the Bayesian regression problem given only the interventional data $\xi_{k-1}\setminus \xi_0=\{(s^\tau_h, a^\tau_h, r^\tau_h)\}_{(\tau, h)\in[k-1]\times[H]}$. 
\subsection{Theory}
The following theorem characterizes the regret of DOVI, which is defined in \eqref{eq::def_regret}.
\begin{theorem}[Regret of DOVI]
\label{thm::regret_backdoor}
Let  $\beta = CdH\sqrt{\log(d(T + nH)/\zeta)}$ and $\lambda = 1$, where $C >0$ and $\zeta\in(0, 1]$ are absolute constants. Under Assumptions \ref{asu::backdoor_crit} and \ref{asu::linMDP_backdoor}, it holds with probability at least $1 - 5\zeta/2$ that
\#
\textrm{Regret}(T) \leq C'\cdot \Delta_H\cdot\sqrt{d^3H^3 T}\cdot \sqrt{\log\bigl(d(T + nH)/\zeta\bigr)},
\#
where $C'>0$ is an absolute constant and
\#\label{eq::def_info_gain_backdoor}
\Delta_H = \frac{1}{\sqrt{dH^2}}\sum^H_{h = 1}\bigl(\log\det(\Lambda^{K+1}_h) - \log\det(\Lambda^1_h)\bigr)^{1/2}.
\#
\end{theorem}
\begin{proof}
See \S\ref{pf::regret_backdoor} for a detailed proof.
\end{proof}
Note that $\Lambda^{K+1}_h\preceq (n+K+\lambda)I$ and $\Lambda^1_h \succeq \lambda I$ for all $h\in[H]$. Hence, it holds that $\Delta_H = \cO(\sqrt{\log(n +K+1)})$ in the worst case. Thus, the regret of DOVI is $\cO(\sqrt{d^3H^3T})$ up to logarithmic factors, which is optimal in the total number of steps $T$ if we only consider the online setting. However, $\Delta_H$ is possibly much smaller than $\cO(\sqrt{\log(n +K+1)})$, depending on the amount of information carried over by the confounded observational data from the offline setting, which is quantified in the following.
\vskip4pt
\noindent{\bf Interpretation of $\Delta_H$: An Information-Theoretic Perspective.}
Let $\omega^*_h$ be the parameter of the globally optimal action-value function $Q^*_h$, which corresponds to $\pi^*$ in \eqref{eq::def_regret}. Recall that we denote by $\xi_0$ and $\xi_{K}$ the confounded observational data $\{(s^i_h, a^i_h, u^i_h, r^i_h)\}_{(i, h)\in[n]\times[H]}$ and the union $\{(s^i_h, a^i_h, u^i_h, r^i_h)\}_{(i, h)\in[n]\times[H]}\cup\{(s^k_h, a^k_h, r^k_h)\}_{(k, h)\in[K]\times[H]}$ of the confounded observational data and the interventional data up to the $K$-th episode, respectively. We consider the aforementioned Bayesian counterpart of the confounded MDP, where the prior of $\omega^*_h$ is also $N(0, \lambda I)$. In such a ``parallel" confounded MDP, we have
\#
\omega^*_h\sim N(0, \lambda I), \quad \omega^*_h\given\xi_{0} \sim N\bigl(\mu^*_{1, h}, (\Lambda^1_h)^{-1}\bigr), \quad \omega^*_h\given\xi_{K} \sim N\bigl(\mu^*_{K, h}, (\Lambda^{K+1}_h)^{-1}\bigr),
\#
where 
\$
\mu^*_{1, h} &= (\Lambda^{1}_h)^{-1}\sum^n_{i = 1} \phi_h(s_h^i, a_h^i, u_h^i) \cdot \bigl(V^*_{h+1}(s^i_{h+1})+ r_h^i\bigr),\notag\\
\mu^*_{K, h} &= (\Lambda^{K+1}_h)^{-1}\biggl(\Lambda^{1}_h \mu^*_{1, h}+ \sum^{K}_{\tau = 1} \psi_h(s^\tau_h, a^\tau_h) \cdot \bigl(V^*_{h+1}(s^\tau_{h+1}) + r_h^\tau\bigr)\biggr).
\$
It then holds for the right-hand side of \eqref{eq::def_info_gain_backdoor} that
\#\label{eq::info_gain_1_to_K}
&1/2\cdot\log\det(\Lambda^{K+1}_h)  - 1/2\cdot\log\det(\Lambda^1_h)=H(\omega^*_h\given\xi_{0}) - H(\omega^*_h\given\xi_{K}).
\#
The left-hand side of \eqref{eq::info_gain_1_to_K} characterizes the information gain of intervention in the online setting given the confounded observational data in the offline setting. In other words, if the confounded observational data are sufficiently informative upon the backdoor adjustment, then $\Delta_H$ is small, which implies that the regret is small. More specifically, the matrices $(\Lambda^{1}_{h})^{-1}$ and $(\Lambda^{K+1}_{h})^{-1}$ defined in \eqref{eq::backdoor_var} characterize the ellipsoidal confidence sets given $\xi_0$ and $\xi_{K}$, respectively. If the confounded observational data are sufficiently informative upon the backdoor adjustment, $\Lambda^{K+1}_{h}$ is close to $\Lambda^{1}_{h}$. To illustrate, let $\{\psi_h(s^\tau_h, a^\tau_h)\}_{(\tau, h)\in[K]\times[H]}$ and $\{\phi_h(s^i_h, a^i_h, u^i_h)\}_{(i, h)\in[n]\times[H]}$ be sampled uniformly at random from the canonical basis $\{e_\ell\}_{\ell\in[d]}$ of $\RR^d$. It then holds that $\Lambda^{K+1}_{h}\approx (K+n)I/d + \lambda I$ and $\Lambda^1_{h}\approx nI/d + \lambda I$. Hence, for $\lambda = 1$ and sufficiently large $n$ and $K$, we have $\Delta_H = \cO(\sqrt{\log(1 + K/(n + d))}) = \cO(\sqrt{K/(n + d)})$. For example, for $n = \Omega(K^2)$, it holds that $\Delta_H = \cO(n^{-1/2})$, which implies that the regret of DOVI is $\cO(n^{-1/2}\cdot\sqrt{d^3H^3T })$. In other words, if the confounded observational data are sufficiently informative upon the backdoor adjustment, the regret of DOVI can be arbitrarily small given a sufficiently large sample size $n$ of the confounded observational data, which is often the case in practice \citep{murphy2003optimal, chakraborty2014dynamic, de2019causal, li2020make, levine2020offline}. 
  

\section{Algorithm and Theory for Unobserved Confounder}
\label{sec::frontdoor_LSVI}

In this section, we extend DOVI to handle the case where the confounders are unobserved in both the online setting and the offline setting. We then characterize the regret of such an extension of DOVI, namely DOVI\textsuperscript{+}. In comparison with DOVI, DOVI\textsuperscript{+} additionally incorporates an intermediate state at each step, which extends the length of each episode from $H$ to $2H$.

\subsection{Algorithm}
\vskip4pt
\noindent{\bf Frontdoor Adjustment.}
Since the confounders $\{w_h\}_{h\in[H]}$ are unobserved in the offline setting, the confounded observational data $\{(s^i_h, a^i_h, r^i_h)\}_{(i, h)\in[n]\times[H]}$ are insufficient for the identification of the causal effect $\PP(s_{h+1}\given s_h, \doo(a_h))$ \citep{pearl2009causality, peters2017elements}. However, such a causal effect is identifiable if we observe the intermediate states $\{m_h\}_{h\in[H]}$ that satisfy the following frontdoor criterion.
\begin{assumption}[Frontdoor Criterion \citep{pearl2009causality, peters2017elements}]
\label{asu::frontdoor_crit}
In the SCM defined in \S\ref{sec::MDP_env}, for all $h\in[H]$, there additionally exists an observed intermediate state $m_h$ that satisfies the frontdoor criterion, that is,
\begin{itemize}
\item $m_h$ intercepts every directed path from $a_h$ to $s_{h+1}$,
\item conditioning on $s_h$, no path between $a_h$ and $m_h$ has an incoming arrow into $a_h$, and
\item conditioning on $s_h$, $a_h$ $d$-separates every path between $m_h$ and $s_{h+1}$ that has an incoming arrow into $m_h$.
\end{itemize}
\end{assumption}
\begin{figure*}[ht!]
    \centering
    \begin{subfigure}[t]{0.5\textwidth}
        \centering
        \includegraphics[height=1.2in]{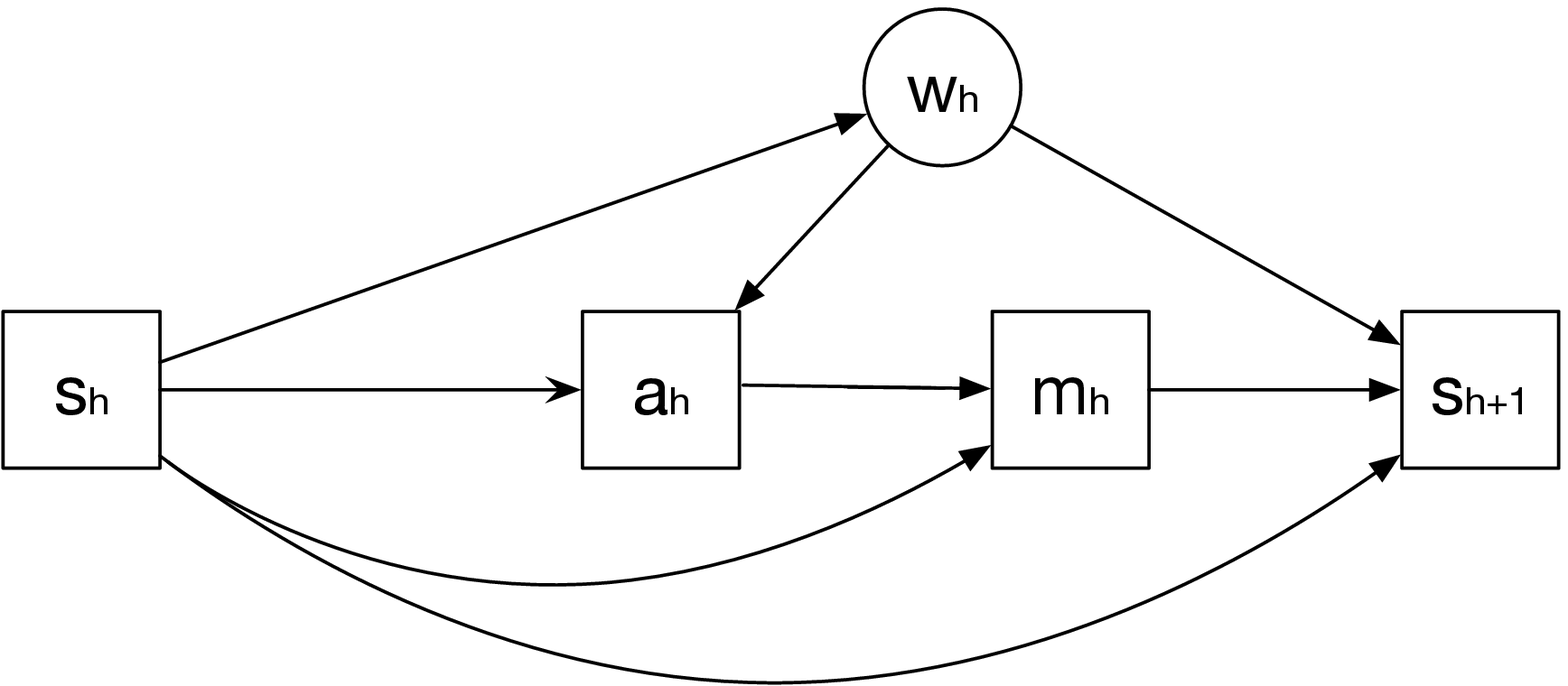}
        \caption{Offline Setting}
    \end{subfigure}%
    ~ 
    \begin{subfigure}[t]{0.5\textwidth}
        \centering
        \includegraphics[height=1.2in]{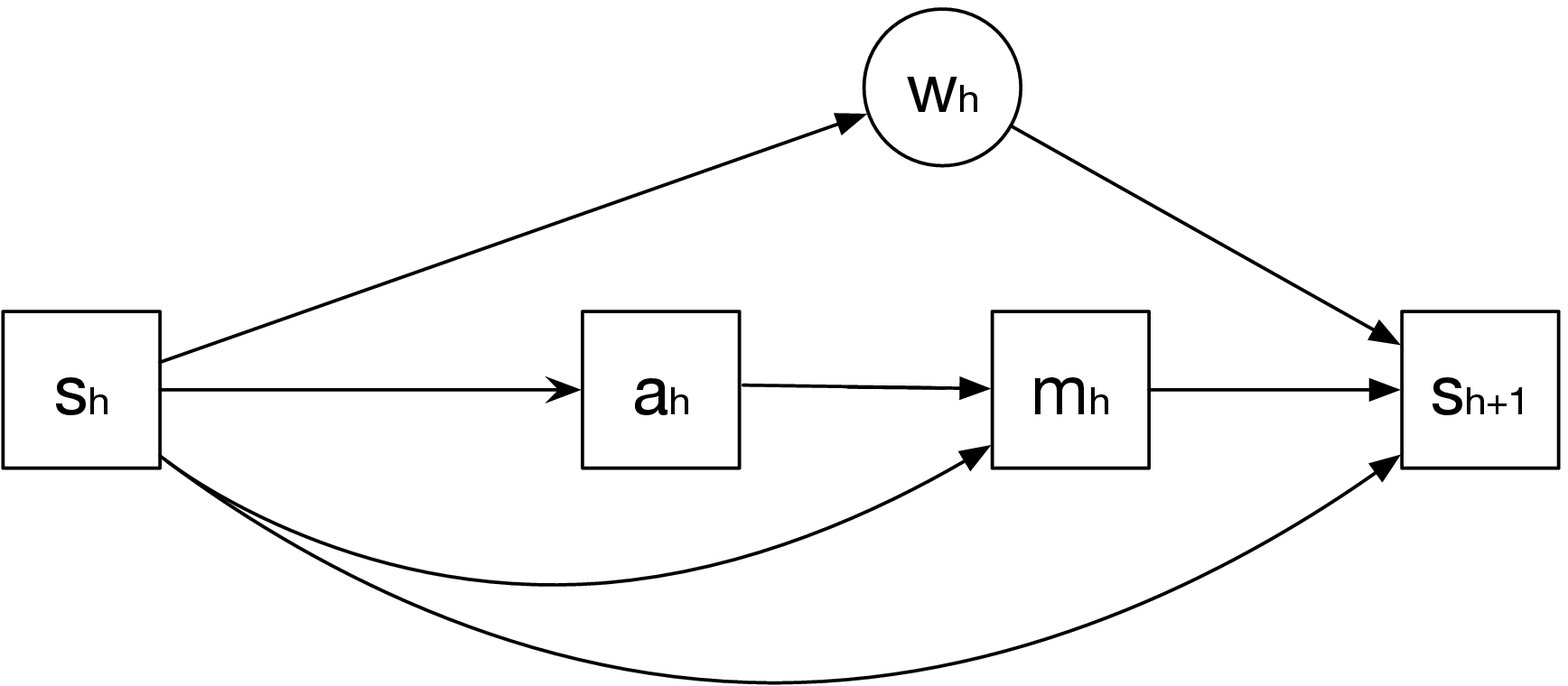}
        \caption{Online Setting}
    \end{subfigure}
    \caption{Causal diagrams of the $h$-th step of the confounded MDP with the intermediate state (a) in the offline setting and (b) in the online setting, respectively.}
    \label{fig::frontdoor_1}
\end{figure*}
\begin{figure*}[ht!]
    \centering
    \includegraphics[height=1.2in]{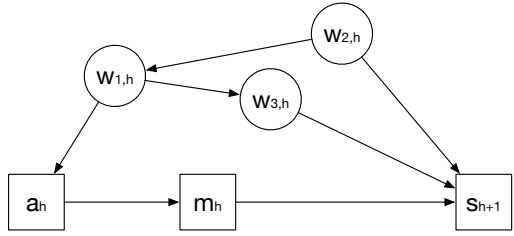}
    \caption{An illustration of the frontdoor criterion. The causal diagram corresponds to the $h$-th step of the confounded MDP conditioning on $s_h$. Here $w_h = \{w_{1, h},w_{2, h},w_{3, h}\}$ is the confounder and the intermediate state $m_h$ satisfies the frontdoor criterion.}
    \label{fig::frontdoor_2}
\end{figure*}

See Figure \ref{fig::frontdoor_1} for the causal diagram that describes such an SCM and Figure \ref{fig::frontdoor_2} for an example that satisfies the frontdoor criterion.
Intuitively, Assumption \ref{asu::frontdoor_crit} ensures that, conditioning on $s_h$, (i) the intermediate state $m_h$ is caused by the action $a_h$ and the causal effect of the action $a_h$ on the next state $s_{h+1}$ is summarized by $m_h$, while (ii) the action $a_h$ and the intermediate state $m_h$ are not confounded. In the sequel, we denote by $\cM$ the space of intermediate states and $\breve\cP_h(\cdot\given\cdot, \cdot)$ the transition kernel that determines $m_h$ given $s_h$ and $a_h$. The causal effect $\PP(s_{h+1}\given s_h, \doo(a_h))$ is identified as follows.
\begin{proposition}[Frontdoor Adjustment \citep{pearl2009causality}]
\label{prop::frontdoor_adj}
Under Assumption \ref{asu::frontdoor_crit}, it holds that
\$
\PP\bigl(s_{h+1}\biggiven s_h, \doo(a_h) \bigr) = \EE_{m_h, a'_h}\bigl[ \PP(s_{h+1}\given s_h, a'_h, m_h)\bigr],
\$
where the expectation $\EE_{m_h, a'_h}$ is taken with respect to $m_h\sim\breve\cP_h(\cdot\given s_h, a_h)$ and $a'_h\sim \EE_{w_h\sim \tilde\cP_h(\cdot\given s_h)}[\nu_h(\cdot\given s_h, w_h)]$. Here $(s_{h+1}, s_h, a_h, m_h)$ follows the SCM define in \S\ref{sec::MDP_env} with the intermediate states $\{m_h\}_{h\in[H]}$ in the offline setting.
\end{proposition}

\vskip4pt
\noindent{\bf Frontdoor-Adjusted Bellman Equation. }
In the sequel, we assume without loss of generality that the reward $r_h$ is deterministic and only depends on the state $s_h$ and the action $a_h$. In parallel to \eqref{eq::bellman_backdoor}, we have
\#\label{eq::bellman_frontdoor}
Q^\pi_h(s_h, a_h) = r_h(s_h, a_h) + \EE_{s_{h+1}}\bigl[V^\pi_{h+1}(s_{h+1})\bigr],
\#
where the expectation $\EE_{s_{h+1}}$ is taken with respect to $s_{h+1}\sim \PP(\cdot\given s_h, \doo(a_h))$. We define the the following transition operators,
\$
&(\PP_{h+1/2} V)(s_h, m_h) = \EE_{ s_{h+1} \sim \PP(\cdot \given s_h, \doo(m_h))}\bigl[V(s_{h+1}) \bigr],\quad \forall V:\cS\mapsto \RR, ~ (s_h, m_h) \in\cS\times\cM,\notag\\
&(\PP_{h} \tilde V)(s_h, a_h) = \EE_{m_h\sim \PP(\cdot\given s_h, \doo(a_h))}\bigl[ \tilde V(s_h, m_h) \bigr], \quad \forall \tilde V:\cS\times\cM\mapsto \RR, ~ (s_h, a_h) \in\cS\times\cA.
\$
We highlight that, under Assumption \ref{asu::frontdoor_crit}, the causal effect $\PP(m_h\given s_h, \doo(a_h))$ coincides with the conditional probability $\PP(m_h\given s_h, a_h)$, since $a_h$ and $m_h$ are not confounded given $s_h$. In the sequel, we define the value function at the intermediate state by $V^\pi_{h+1/2}(s_h, m_h) = (\PP_{h+1/2} V^{\pi}_{h+1})(s_h, m_h)$. We have the following Bellman equation,
\#\label{eq::frontdoor_bellman_0}
Q^\pi_h(s_h, a_h) &= r_h(s_h, a_h) + \bigl(\PP_{h}(\PP_{h+1/2} V^\pi_{h+1})\bigr)(s_h, a_h)\notag\\
&= r_h(s_h, a_h) + (\PP_{h} V^\pi_{h+1/2})(s_h, a_h).
\#
Correspondingly, the Bellman optimality equation takes the following form,
\#\label{eq::frontdoor_bellman}
&Q^*_h(s_h, a_h)  = r_h(s_h, a_h) + (\PP_{h} V^*_{h+1/2})(s_h, a_h),\notag\\
 &V^*_{h+1/2}(s_h, m_h) = (\PP_{h+1/2} V^*_{h+1})(s_h, m_h),  \quad V^*_{h}(s_h) = \max_{a_h\in\cA} Q^*_h(s_h, a_h).
\#

\vskip4pt
\noindent{\bf Linear Function Approximation.}
In parallel to Assumption \ref{asu::linMDP_backdoor}, we focus on the following setting with linear transition kernels and reward functions \citep{yang2019sample, yang2019reinforcement, jin2019provably, cai2019provably}, which corresponds to a linear SCM \citep{peters2017elements}.
\begin{assumption}[Linear Confounded MDP]
\label{asu::linMDP_frontdoor}
We assume that
\$
\cP_h(s_{h+1} \given s_h, m_h, w_h) &= \langle \rho_{h}(s_h, m_h, w_h), \mu_h(s_{h+1}) \rangle,\quad \forall h\in[H],~(s_h, m_h, w_h)\in\cS\times\cM\times\cW,\\
\breve\cP_h(m_h \given s_h, a_h) &= \langle \gamma_{h}(s_h, a_h), \overline\mu_h(m_{h}) \rangle,\quad\forall h\in[H],~(m_h, s_h, a_h)\in\cM\times\cS\times\cA.
\$
where $\rho_h(\cdot, \cdot, \cdot)$, $\gamma_h(\cdot, \cdot)$, $\mu_h(\cdot) = (\mu_{1, h}(\cdot), \ldots, \mu_{d, h}(\cdot))^\top$, and $\overline \mu_h(\cdot)= (\overline\mu_{1, h}(\cdot), \ldots, \overline\mu_{d, h}(\cdot))^\top$ are $\RR^d$-valued functions. We assume that $\|\rho_h(s_h, m_h, w_h)\|_2 \leq 1$,  $\|\gamma_h(s_h, a_h)\|_2 \leq 1$, $\sum^d_{i = 1}\|\mu_{i, h}\|^2_1 \leq d$, and $\sum^d_{i = 1}\|\overline\mu_{i, h}\|^2_1 \leq d$ for all $h\in[H]$ and $(s_h, a_h, m_h, w_h)\in\cS\times\cA\times\cM\times\cW$. 
Meanwhile, we assume that
\$
r_h(s_h, a_h) = \gamma_h(s_h, a_h)^\top \theta_h, \quad \forall (h, k)\in[H]\times[K],
\$
where $\theta_h \in \RR^d$ and $\|\theta_h\|_2\leq \sqrt{d}$ for all $h\in[H]$.
\end{assumption}
\begin{proposition}
\label{prop::frontdoor_prop}
We define $\tilde\nu_h(a_h\given s_h) = \EE_{w_h\sim\tilde\cP_h(\cdot\given s_h)}[\nu_h(a_h\given s_h, w_h)]$, where $\nu = \{\nu_h\}_{h\in[H]}$ is the behavior policy. With a slight abuse of notation, we define the frontdoor-adjusted feature as follows,
\#\label{eq::def_phi_front}
\phi_h(s_h, a_h, m_h) = \frac{\EE_{w_h \sim \tilde\cP_h(\cdot\given s_h)}\bigl[ \rho_h(s_h, m_h, w_h)\cdot \nu_h(a_h \given s_h, w_h) \bigr]}{\tilde \nu_h(a_h\given s_h)}, \quad \forall h \in [H].
\#
Under Assumption \ref{asu::linMDP_frontdoor}, it holds that
\#\label{eq::def_cond_frontdoor}
\PP(s_{h+1}\given s_h, a_h, m_h) = \langle \phi_h(s_h, a_h, m_h), \mu_h(s_{h+1})\rangle.
\#
\end{proposition}
\begin{proof}
See \S\ref{sec::pf_phi_cond} for a detailed proof.
\end{proof}

\begin{algorithm}[htpb]
\caption{DOVI\textsuperscript{+} for Confounded MDP.}
\begin{algorithmic}[1]\label{alg::frontdoor}
\REQUIRE Observational data $\{(s_{h}^i, a_{h}^i, m_h^i, r^i_h)\}_{i \in [n], h \in [H]}$, tuning parameters $\lambda, \beta >0$, features $\{\phi_h\}_{h\in[H]}$ and $\{\psi_h\}_{h\in[H]}$, which are defined in \eqref{eq::def_phi_front} and \eqref{eq::psi_frontdoor}, respectively.
\STATE{\bf Initialization:} Set $\{Q^0_h, V^0_{h+1/2},V^0_h\}_{h\in[H]}$ as zero functions and $V^k_{H+1}$ as a zero function for $k\in[K]$. 
\FOR{$k = 1, \ldots, K$}
\FOR{$h = H, \ldots, 1$}
\STATE {\bf Update $V^k_{h+1/2}$:}
\STATE \label{line::mechanism_frontdoor}Set $\omega^k_{1, h} \leftarrow \argmin_{\omega\in\RR^d} \sum^{k-1}_{\tau = 1}( V^\tau_{h+1}(s^\tau_{h+1})-\omega^\top \psi_h(s^\tau_h, m^\tau_h))^2 + \lambda \|\omega\|^2_2 + L^k_{1, h}(\omega)$, where $L^k_{1, h}$ is defined in \eqref{eq::reg_frontdoor}.
\STATE \label{line::V_half} Set $V^k_{h+1/2}(s_h, m_h) \leftarrow  \min\{\psi_h(s_h, m_h)^\top \omega^k_{1,h} + \Gamma^k_{h+1/2}(s_h, m_h), H-h\}$ for all $(s_h, m_h)\in\cS\times\cM$, where $\Gamma^k_{h+1/2}$ is defined in \eqref{eq::UCB_frontdoor}.
\STATE {\bf Update $Q^k_{h}$:}
\STATE Set $\omega^k_{2, h} \leftarrow\argmin_{\omega\in\RR^d} \sum^{k-1}_{\tau = 1}( r^k_h + V^k_{h+1/2}(s^\tau_{h}, m^\tau_h)-\omega^\top \gamma_h(s^\tau_h, a^\tau_h))^2 + \lambda \|\omega\|^2_2 + L^k_{2, h}(\omega)$, where $L^k_{2, h}$ is defined in \eqref{eq::LSVI_frontdoor_11}.
\STATE \label{line::Q_update} Set $Q^k_h(s_h, a_h) \leftarrow \min\{\gamma_h(s_h, a_h)^\top \omega^k_{2, h} + \Gamma^k_{h}(s_h, a_h), H-h\}$ for all $(s_h, a_h)\in\cS\times\cA$, where $\Gamma^k_{h}$ is defined in \eqref{eq::UCB_frontdoor_1}.
\STATE {\bf Update $\pi^k_h$ and $V^k_h$:}
\STATE Set $\pi^k_h(\cdot\given s_h) \leftarrow \argmax_{a_h\in\cA} Q^k_h(s_h, a_h)$ for all $s_h\in\cS$.
\STATE \label{line::V_next_frontdoor}Set $V^k_h(\cdot) \leftarrow \langle \pi^k_h(\cdot\given\cdot),  Q^k_h(\cdot, \cdot)\rangle_\cA$.
\ENDFOR
\STATE Obtain $s^k_1$ from the environment.
\FOR{$h = 1, \ldots, H$}
\STATE Take $a^k_h\sim \pi^k_h(\cdot\given s^k_h)$. Obtain $r^k_h = r_h(s^k_h, a^k_h)$, $m^k_{h}$, and $s^k_{h+1}$.
\ENDFOR
\ENDFOR
\end{algorithmic}
\end{algorithm}

\vskip4pt
\noindent{\bf DOVI\textsuperscript{+}: Update of $V^k_{h+1/2}$.}
With a slight abuse of notation, we define the following feature,
\#\label{eq::psi_frontdoor}
\psi_{h}(s_h, m_h) = \EE_{w_h\sim\tilde\cP_h(\cdot\given s_h)}\bigl[\rho_{h}(s_h, m_h, w_h) \bigr].
\#
Conditioning on the state $s_h$, the confounder $w_h$ satisfies the backdoor criterion for identifying the causal effect $\PP(s_{h+1}\given s_h, \doo(m_h))$, although it is unobserved. In the sequel, we assume that either the density of $\{\tilde \cP_h(\cdot\given s_h)\}_{h\in[H]}$ is known to us or the features $\{\phi_h\}_{h\in[H]}$ and $\{\psi_h\}_{h\in[H]}$ are known to us. Following from \eqref{eq::psi_frontdoor}, Proposition \ref{lem::backdoor}, and Assumption \ref{asu::linMDP_frontdoor}, it holds for all $h\in[H]$ and $(s_{h+1}, s_h, m_h)\in\cS\times\cS\times\cM$ that 
\#\label{eq::front_back_adjust}
\PP\bigl(s_{h+1}\biggiven s_h, \doo(m_h)\bigr) = \langle \psi_h(s_h, m_h), \mu_h(s_{h+1}) \rangle.
\# 
Hence, by the Bellman equation and the Bellman optimality equation in \eqref{eq::frontdoor_bellman_0} and \eqref{eq::frontdoor_bellman}, respectively, the value functions at the intermediate state $V^\pi_{h+1/2}$ and $V^*_{h+1/2}$ are linear in the feature $\psi_h$ for all $\pi$. To solve for $V^*_{h+1/2}$ in the Bellman optimality equation in \eqref{eq::frontdoor_bellman}, we minimize the following empirical mean-squared Bellman error as follows at each step,
\#\label{eq::LSVI_frontdoor}
\omega^k_{1, h} &\leftarrow \argmin_{\omega\in\RR^d} \sum^{k-1}_{\tau = 1}\bigl( V^\tau_{h+1}(s^\tau_{h+1})-\omega^\top \psi_h(s^\tau_h, m^\tau_h)\bigr)^2 + \lambda \|\omega\|^2_2 + L^k_{1, h}(\omega),\quad h = H, \ldots, 1,
\#
where we set $V^{k}_{H+1} = 0$ for all $k\in[K]$ and $V^\tau_{h+1}$ is defined in Line \ref{line::V_next_frontdoor} of Algorithm \ref{alg::frontdoor} for all $(\tau, h)\in[K]\times[H-1]$. Here $k$ is the index of episode, $\lambda >0$ is a tuning parameter, and $L^k_{1, h}$ is a regularizer, which is constructed based on the confounded observational data. More specifically, we define
\#\label{eq::reg_frontdoor}
L^k_{1, h}(\omega) = \sum^{n}_{i= 1}\bigl( V^\tau_{h+1}(s^i_{h+1})-\omega^\top \phi_h(s^i_h, a^i_h, m^i_h)\bigr)^2, \quad \forall (k, h)\in[K]\times[H],
\#
which corresponds to the least-squares loss for regressing $V^\tau_{h+1}(s^i_{h+1})$ against $\phi_h(s^i_h, a^i_h, m^i_h)$ for all $i\in[n]$. Here $\{(s_{h}^i, a_{h}^i, m_h^i, r_h^i)\}_{(i, h)\in[n]\times[H]}$ are the confounded observational data, where $s^i_{h+1}\sim \cP_h(\cdot\given s^i_h, a^i_h, w^i_h)$, $m^i_h\sim \breve\cP_h(\cdot\given s^i_h, a^i_h)$, and $a^i_h\sim\nu_h(\cdot\given s^i_h, w^i_h)$ with $\nu = \{\nu_h\}_{h\in[H]}$ being the behavior policy.

The update in \eqref{eq::LSVI_frontdoor} takes the following explicit form,
\#\label{eq::frontdoor_update_1}
\omega^k_{1, h} \leftarrow (\Lambda^{k}_{1, h})^{-1} \biggl( &\sum^{k-1}_{\tau = 1} \psi_h(s^\tau_h, m^\tau_h) \cdot V^k_{h+1}(s^\tau_{h+1}) + \sum^n_{i = 1} \phi_h(s_h^i, a_h^i, m_h^i) \cdot V^k_{h+1}(s^i_{h+1})\biggr),
\#
where 
\#\label{eq::frontdoor_var_1}
\Lambda^k_{1, h} = \sum^{k-1}_{\tau = 1} \psi_h(s^\tau_h, m^\tau_h) \psi_h(s^\tau_h, m^\tau_h)^\top+ \sum^n_{i = 1} \phi_h(s_h^i, a_h^i, m_h^i)\phi_h(s_h^i, a_h^i, m_h^i)^\top + \lambda I.
\#
Meanwhile, we employ the following UCB of $\psi_h(s^k_h, m^k_h)^\top\omega^k_{1, h}$ for all $(s^k_h, m^k_h)\in\cS\times\cM$,
\#\label{eq::UCB_frontdoor}
\Gamma^k_{h + 1/2}(s^k_h, m^k_h) = \beta\cdot\Bigl(\log \det\bigl(\Lambda^k_{1, h} + \psi_h(s^k_h, m^k_h)\psi_h(s^k_h, m^k_h)^\top\bigr) - \log\det(\Lambda^k_{1, h})\Bigr)^{1/2}.
\#
The update of $V^k_{h +1/2}$ is defined in Line \ref{line::V_half} of Algorithm \ref{alg::frontdoor}. 
\vskip4pt
\noindent{\bf DOVI\textsuperscript{+}: Update of $Q^k_h$.}
Upon obtaining $V^k_{h+1/2}$, we solve for $Q^k_h$ by minimizing the following empirical mean-squared Bellman error as follows at each step,
\#\label{eq::LSVI_frontdoor_1}
\omega^k_{2, h} \leftarrow\argmin_{\omega\in\RR^d} &\sum^{k-1}_{\tau = 1}\bigl( r^k_h + V^k_{h+1/2}(s^\tau_{h}, m^\tau_h)-\omega^\top \gamma_h(s^\tau_h, a^\tau_h)\bigr)^2 \notag\\
&+ \lambda \|\omega\|^2_2 + L^k_{2, h}(\omega), \quad h = H, \ldots, 1.
\#
Here $L^k_{2, h}$ is a regularizer, which is defined as follows,
\#\label{eq::LSVI_frontdoor_11}
L^k_{2, h}(\omega) = \sum^{n}_{i= 1}\bigl( r^i_h + V^k_{h+1/2}(s^i_{h}, m^i_h)-\omega^\top \gamma_h(s^i_h, a^i_h)\bigr)^2, \quad \forall (k, h)\in[K]\times[H].
\#
The update in \eqref{eq::LSVI_frontdoor_1} takes the following explicit form,
\$
\omega^k_{2, h} \leftarrow(\Lambda^{k}_{2, h})^{-1}  \biggl(& \sum^{k-1}_{\tau = 1} \gamma_h(s^\tau_h, a^\tau_h) \cdot \bigl(V^k_{h+1/2}(s^\tau_{h}, m^\tau_h) + r_h^\tau\bigr) + \sum^n_{i = 1} \gamma_h(s_h^i, a_h^i)\cdot \bigl(V^k_{h+1/2}(s^i_{h}, m^i_{h}) + r_h^i\bigr)\biggr),
\$
where 
\$
\Lambda^k_{2, h} = \sum^{k-1}_{\tau = 1} \gamma_h(s^\tau_h, a^\tau_h) \gamma_h(s^\tau_h, a^\tau_h)^\top+ \sum^n_{i = 1} \gamma_h(s_h^i, a_h^i)\gamma_h(s_h^i, a_h^i)^\top + \lambda I.
\$
We employ the following UCB of $\gamma_h(s_h^k, a_h^k)^\top\omega^k_{2, h}$ for all $(s^k_h, a^k_h)\in\cS\times\cA$,
\#\label{eq::UCB_frontdoor_1}
\Gamma^k_{h}(s^k_h, a^k_h) = \beta\cdot\Bigl(\log \det\bigl(\Lambda^k_{2, h} + \gamma_h(s^k_h, a^k_h)\gamma_h(s^k_h, a^k_h)^\top\bigr) - \log\det(\Lambda^k_{2, h})\Bigr)^{1/2}.
\#
The update of $Q^k_{h}$ is defined in Line \ref{line::Q_update} of Algorithm \ref{alg::frontdoor}.

\subsection{Theory}
In parallel to Theorem \ref{thm::regret_backdoor}, the following theorem characterizes the regret of DOVI\textsuperscript{+}, which is defined in \eqref{eq::def_regret}
\begin{theorem}[Regret of DOVI\textsuperscript{+}]
\label{thm::regret_frontdoor}
Let $\beta = CdH\sqrt{\log(d(T + nH)/\zeta)}$ and $\lambda = 1$, where $C >0$ and $\zeta\in(0, 1]$ are absolute constants. Under Assumptions \ref{asu::frontdoor_crit} and \ref{asu::linMDP_frontdoor}, it holds with probability at least $1 - 5\zeta$ that
\$
\textrm{Regret}(T) \leq C'\cdot (\Delta_{1, H} + \Delta_{2, H})\cdot \sqrt{d^3H^3T}\cdot \sqrt{\log\bigl(d(T + nH)/\zeta\bigr)},
\$
where $C'>0$ is an absolute constant and
\$
\Delta_{1, H} =\frac{1}{\sqrt{dH^2}}\sum^H_{h = 1}\bigl(\log \det(\Lambda^{K+1}_{1, h}) - \log\det(\Lambda^1_{1, h})\bigr)^{1/2},\\
\Delta_{2, H} =\frac{1}{\sqrt{dH^2}}\sum^H_{h = 1}\bigl(\log \det(\Lambda^{K+1}_{2, h}) - \log\det(\Lambda^1_{2, h})\bigr)^{1/2}.
\$
\end{theorem}
\begin{proof}
See \S\ref{pf::regret_frontdoor} for a detailed proof.
\end{proof}
See the discussion of Theorem \ref{thm::regret_backdoor} in \S\ref{sec::backdoor_LSVI}, where $\Delta_H$ corresponds to $\Delta_{1, H}$ and $\Delta_{2, H}$ in Theorem \ref{thm::regret_frontdoor}. In particular, $\Delta_{1, H}$ and $\Delta_{2, H}$ admit the same information-theoretic interpretation.

\section{Mechanism of Utilizing Confounded Observational Data}
\label{sec::mechanism}
In this section, we discuss the mechanism of incorporating the confounded observational data. 

\subsection{Partially Observed Confounder}
Corresponding to Line \ref{line::mechanism_backdoor}  of Algorithm \ref{alg::backdoor}, DOVI effectively estimates the causal effect $\PP(\cdot\given s_h, \doo(a_h))$ using
\#\label{eq::intuition_backdoor}
 \psi_h(s_h, a_h)^\top(\Lambda^k_h)^{-1} \biggl(\sum^{k-1}_{\tau = 1}\psi_h(s^\tau_h, a^\tau_h)\cdot \delta_{s^\tau_{h+1}}(\cdot) + \sum^n_{i = 1} \phi_h(s_h^i, a_h^i, u_h^i) \cdot\delta_{s^i_{h+1}}(\cdot)\biggr),
\#
where we denote by $\delta_s(\cdot)$ the Dirac measure at $s$. To see why it works, let the tuning parameter $\lambda$ be sufficiently small. By the definition of $\Lambda^k_h$ in \eqref{eq::backdoor_var}, we have
\#\label{eq::intuition_backdoor_1}
&\PP\bigl(\cdot\biggiven s_h, \doo(a_h)\bigr) = \langle \psi_h(s_h, a_h), \mu_h(\cdot)\rangle \notag\\
&\qquad\approx \psi_h(s_h, a_h)^\top (\Lambda^k_h)^{-1}\biggl(\sum^{k-1}_{\tau = 1}\psi_h(s^\tau_h, a^\tau_h)\cdot \langle \psi_h(s^\tau_h, a^\tau_h), \mu_h(\cdot)\rangle\notag\\
&\qquad\qquad\qquad\qquad\qquad\qquad\quad + \sum^n_{i = 1} \phi_h(s_h^i, a_h^i, u_h^i)\cdot \langle \phi_h(s_h^i, a_h^i, u_h^i), \mu_h(\cdot)\rangle\biggr).
\#
Meanwhile, Assumption \ref{asu::linMDP_backdoor} and Proposition \ref{prop::backdoor_feature} imply
\$
&\PP\bigl(\cdot\biggiven s_h, \doo(a_h)\bigr) = \langle \psi_h(s_h, a_h), \mu_h(\cdot)\rangle, \notag\\
&\cP_h(\cdot\given s_h, a_h, u_h) = \langle\phi_h(s_h, a_h, u_h), \mu_h(\cdot)\rangle,
\$
which rely on the backdoor adjustment. Since $s^\tau_{h+1}$ and $s^i_{h+1}$ in \eqref{eq::intuition_backdoor} are sampled following $\PP(\cdot\given s^\tau_h, \doo(a^\tau_h))$ and $\cP_h(\cdot\given s^i_h, a^i_h, u^i_h)$, respectively, \eqref{eq::intuition_backdoor} approximates the right-hand side of 
\eqref{eq::intuition_backdoor_1} as its empirical version. As $k, n\to+\infty$, \eqref{eq::intuition_backdoor} converges to the right-hand side of \eqref{eq::intuition_backdoor_1} as well as the causal effect $\PP(\cdot\given s_h, \doo(a_h))$.


\subsection{Unobserved Confounder}
If the confounders $\{w_h\}_{h\in[H]}$ are unobserved in the offline setting, the backdoor adjustment in \S\ref{sec::backdoor_LSVI} is not applicable. Alternatively, the intermediate states $\{m_h\}_{h\in[H]}$ allow us to estimate the causal effect without observing the confounders. The key is that the frontdoor criterion in Assumption \ref{asu::frontdoor_crit} implies
\#\label{eq::frontdoor_intuition_0}
\PP\bigl(s_{h+1}\biggiven s_h, \doo(a_h)\bigr) = \int_{\cM} \PP\bigl(s_{h+1}\biggiven s_h, \doo(m_h)\bigr) \cdot \PP\bigl(m_{h}\biggiven s_h, \doo(a_h)\bigr) \ud m_h.
\#
It remains to estimate $\PP(s_{h+1}\given s_h, \doo(m_h))$ and $\PP(m_{h}\given s_h, \doo(a_h))$ on the right-hand side of \eqref{eq::frontdoor_intuition_0}. Since $a_h$ and $m_h$ are not confounded given $s_h$, the causal effect $\PP(m_{h}\given s_h, \doo(a_h))$ coincides with the conditional distribution $\PP(m_{h}\given s_h, a_h)$, which can be estimated based on the observational data. To estimate the causal effect $\PP(s_{h+1}\given s_h, \doo(m_h))$, we utilize the backdoor adjustment in Proposition \ref{lem::backdoor} with $u_h$ replaced by $a_h$, which is enabled by Assumption \ref{asu::frontdoor_crit}. More specifically, it holds that
\#\label{eq::frontdoor_adjust}
\PP\bigl(s_{h+1}\biggiven s_h, \doo(m_h)\bigr) = \EE_{a'_h \sim \PP(\cdot\given s_h)}\bigl[\cP_h(s_{h+1}\biggiven s_h, a'_h, m_h)\bigr].
\#

Correspondingly, we construct the value function at the intermediate state $V_{h+1/2}$ and adapt the value iteration following the Bellman optimality equation in \eqref{eq::frontdoor_bellman}. To estimate the value functions $\{V^k_{h+1/2}\}_{h\in[H]}$ based on the confounded observational data, we utilize the adjustment in \eqref{eq::frontdoor_adjust}. Corresponding to Line \ref{line::mechanism_frontdoor} of Algorithm \ref{alg::frontdoor}, DOVI\textsuperscript{+} effectively estimates the causal effect $\PP(\cdot\given s_h, \doo(m_h))$ using
\#\label{eq::est_frontdoor}
\psi_h(s_h, m_h)^\top(\Lambda^k_{1, h})^{-1} \biggl(\sum^{k-1}_{\tau = 1}\psi_h(s^\tau_h, m^\tau_h)\cdot \delta_{s^\tau_{h+1}}(\cdot) + \sum^n_{i = 1} \phi_h(s_h^i, a_h^i, m_h^i) \cdot\delta_{s^i_{h+1}}(\cdot)\biggr),
\#
To see why it works, let the tuning parameter $\lambda$ be sufficiently small. By the definition of $\Lambda^k_{1, h}$ in \eqref{eq::frontdoor_var_1}, we have
\#\label{eq::est_frontdoor_exact}
&\PP\bigl(\cdot\biggiven s_h, \doo(m_h)\bigr) = \langle \psi_h(s_h, m_h), \mu_h(\cdot)\rangle \notag\\
&\qquad\approx \psi_h(s_h, m_h)^\top (\Lambda^k_{1, h})^{-1}\biggl(\sum^{k-1}_{\tau = 1}\psi_h(s^\tau_h, m^\tau_h)\cdot \langle \psi_h(s^\tau_h, m^\tau_h), \mu_h(\cdot)\rangle\notag\\
&\qquad\qquad\qquad\qquad\qquad\qquad\qquad + \sum^n_{i = 1} \phi_h(s_h^i, a_h^i, m_h^i)\cdot \langle \phi_h(s_h^i, a_h^i, m_h^i), \mu_h(\cdot)\rangle\biggr).
\#
Meanwhile, Assumption \ref{asu::linMDP_frontdoor} and Proposition \ref{prop::frontdoor_prop} imply
\$
&\PP\bigl(\cdot\biggiven s_h, \doo(m_h)\bigr) = \langle \psi_{h}(s_h, m_h), \mu_h(\cdot)\rangle,\notag\\
&\PP(\cdot\given s_h, a_h, m_h) = \langle \phi_h(s_h, a_h, m_h), \mu_h(\cdot)\rangle.
\$
Since $s^\tau_{h+1}$ and $s^i_{h+1}$ in \eqref{eq::est_frontdoor_exact} are sampled following $\PP(\cdot\given s^\tau_h, \doo(m^\tau_h))$ and $\PP(\cdot\given s^i_h, a^i_h, m^i_h)$, respectively, \eqref{eq::est_frontdoor} approximates the right-hand side of \eqref{eq::est_frontdoor_exact} as its empirical version. As $k, n\to +\infty$, \eqref{eq::est_frontdoor} converges to the right-hand side of \eqref{eq::est_frontdoor_exact} as well as the causal effect $\PP(\cdot\given s_h, \doo(m_h))$.


\bibliographystyle{ims}
\bibliography{CRL}

\newpage
\appendix

\section{Proof of Main Result}

\subsection{Proof of Proposition \ref{prop::backdoor_feature}}
\label{pf::backdoor_feature}
\begin{proof}
Following from Assumption \ref{asu::linMDP_backdoor} and Proposition \ref{lem::backdoor}, it holds for all $(s_h, a_h)\in\cS\times\cA$ that
\$
\PP\bigl(s_{h+1}\given s_h, \doo(a_h)\bigr) &= \EE_{u_{h}\sim \tilde\cP_h(\cdot\given s_h)}\bigl[\cP_h(\cdot\given s_h, a_h, u_h)\bigr] = \EE_{u_{h}\sim \tilde\cP_h(\cdot\given s_h)}\bigl[\langle \phi_h(s_h, a_h, u_h), \mu_h(s_{h+1}) \rangle\bigr]\notag\\
&=\langle\psi_h(s_h, a_h), \mu_h(s_{h+1})\rangle,
\$
where
\$
\psi_h(s_h, a_h) = \EE_{u_h\sim\tilde \cP_h(\cdot\given s_h)}\bigl[\phi_h(s_h, a_h, u_h)\bigr], \quad \forall (s_h, a_H)\in\cS\times\cA.
\$
Similarly, following from Assumption \ref{asu::linMDP_backdoor} and Proposition \ref{lem::backdoor}, it holds for all $(s_h, a_h)\in\cS\times\cA$ that
\$
R_h(s_h, a_h) =\EE\bigl[r_h\biggiven s_h, \doo(a_h)\bigr]= \EE_{u_{h}\sim \tilde\cP_h(\cdot\given s_h)}\bigl[\phi_h(s_h, a_h, u_h)^\top\theta_h\bigr]=\psi_h(s_h, a_h)^\top\theta_h.
\$
Hence, following from the Bellman equations in \eqref{eq::bellman_backdoor} and \eqref{eq::bellman_backdoor_opt}, the action-value functions $Q^\pi_h$ and $Q^*_h$ are linear in the backdoor-adjusted feature $\psi_h$ for all $\pi$. Thus, we complete the proof of Proposition \ref{prop::backdoor_feature}.
\end{proof}

\subsection{Proof of Proposition \ref{prop::frontdoor_prop}}
\label{sec::pf_phi_cond}
\begin{proof}
It holds for all $h\in[H]$ and $(s_{h+1}, s_h, a_h, m_h)\in\cS\times\cS\times\cA\times\cM$ that
\$
&\PP(s_{h+1}, s_h, a_h, m_h) \notag\\
&\qquad=\int_{\cW} \cP_h(s_{h+1}\given s_h, a_u, w_h) \cdot \nu_h(a_h\given s_h, w_h)\cdot \tilde\cP_h(w_h\given s_h)\cdot\breve\cP_h(m_h\given s_h, a_h) \cdot \PP(s_h) \ud w_h.
\$
Meanwhile, it holds for all $h\in[H]$ and $(s_h, a_h, m_h)\in\cS\times\cA\times\cM$ that
\$
\PP(s_h, a_h, m_h) = \int_{\cW}\nu_h(a_h\given s_h, w_h)\cdot \tilde\cP_h(w_h\given s_h)\cdot\breve\cP_h(m_h\given s_h, a_h) \cdot \PP(s_h) \ud w_h.
\$
Hence, we have
\#\label{eq::pf_phi_eq1}
\PP(s_{h+1}\given s_h, a_h, m_h) &= \frac{\PP(s_{h+1}, s_h, a_h, m_h)}{\PP(s_h, a_h, m_h)}\notag\\
&= \frac{\int_{\cW} \cP_h(s_{h+1}\given s_h, a_u, w_h) \cdot \nu_h(a_h\given s_h, w_h)\cdot \tilde\cP_h(w_h\given s_h)  \ud w_h}{\int_{\cW}\nu_h(a_h\given s_h, w_h)\cdot \tilde\cP_h(w_h\given s_h)\ud w_h}.
\#
Meanwhile, following from Assumption \ref{asu::linMDP_frontdoor}, we have
\#\label{eq::pf_phi_eq2}
\cP_h(s_{h+1}\given s_h, a_h, w_h) = \langle\rho_h(s_h, a_h, w_h), \mu_h(s_{h+1})\rangle.
\#
Recall that we define $\tilde \nu_h(a_h\given s_h) = \EE_{w_h\sim\tilde\cP_h(\cdot\given s_h)}[\pi(a_h\given s_h, u_h)]$. Hence, by plugging \eqref{eq::pf_phi_eq2} into \eqref{eq::pf_phi_eq1}, we obtain that
\$
\PP(s_{h+1}\given s_h, a_h, m_h) = \langle\phi_h(s_h, a_h, m_h), \mu_h(s_{h+1})\rangle,
\$
where we define for all $h\in[H]$ and $(s_h, a_h, m_h)\in\cS\times\cA\times\cM$ that
\$
\phi_h(s_h, a_h, m_h) &= \frac{\int_{\cW} \rho_h(s_h, a_u, w_h) \cdot \nu_h(a_h\given s_h, w_h)\cdot \tilde\cP_h(w_h\given s_h)  \ud w_h}{\int_{\cW}\nu_h(a_h\given s_h, w_h)\cdot \tilde\cP_h(w_h\given s_h)\ud w_h}\notag\\
&= \frac{\EE_{w_h \sim \tilde\cP_h(\cdot\given s_h)}\bigl[ \rho_h(s_h, m_h, w_h)\cdot \nu_h(a_h \given s_h, w_h) \bigr]}{\tilde \nu_h(a_h\given s_h)}.
\$
Thus, we complete the proof of Proposition \ref{prop::frontdoor_prop}.
\end{proof}

\subsection{Proof of Theorem \ref{thm::regret_backdoor}}
\label{pf::regret_backdoor}
\begin{proof}
We first define for all $(k, h)\in[K]\times[H]$ the model prediction error $\iota^k_h$ as follows,
\#\label{eq::pf_def_est_err}
\iota^k_h(s_h, a_h) &= -Q^k_h(s_h, a_h) + R_h(s_h, a_h) + (\PP_h V^k_{h+1})(s_h, a_h), \quad \forall (s_h, a_h)\in\cS\times\cA.
\#
We define the filtrations associated with Algorithm \ref{alg::backdoor} as follows.
\begin{definition}[Filtration]
\label{def::filtration_backdoor}
For all $(k, h)\in[K]\times[H]$, we define $\cF_{k, h, 1}$ the $\sigma$-algebra generated by the following set,
\#
B_{k, h, 1} = &\bigl\{(s_{h}^i, a_{h}^i, u_h^i, r^i_h)\bigr\}_{(i , h) \in[n]\times [H]}\cup\bigl\{(s_j^\tau, a_j^\tau, r^\tau_j)\bigr\}_{(\tau, j)\in[k-1]\times[H]}\notag\\
&\qquad\cup\bigl\{(s_j^k, a_j^k, r_j^k)\bigr\}_{j\in[h-1]}\cup\bigl\{(s^k_h, a^k_h)\bigr\}.
\#
Similarly, we define $\cF_{k, h, 2}$ the $\sigma$-algebra generated by the following set, 
\#
B_{k, h, 2} = B_{k, h, 1}\cup\{s^k_{h+1}\}\cup\{r^k_h\}.
\#
Moreover, we define $\cF_{0, h, 2}$ the $\sigma$-algebra generated by the set $\{(s_{h}^i, a_{h}^i, u_h^i, r^i_h)\}_{(i, h) \in[n]\times [H]}$ for all $h\in[H]$. We define the timestep index as follows,
\#\label{eq::timestep}
t(k, h, m) = 2H\cdot k +2(h-1) + m.
\#
It then holds for $t(k, h, m) \leq t(k', h', m')$ that $\cF_{k, h, m} \subseteq \cF_{k', h', m'}$. Hence, the set of $\sigma$-algebra $\{\cF_{k, h, m}\}_{(k, h, m)\in[K]\times [H]\times [2]}$ is a filtration with the timestep index $t(\cdot, \cdot, \cdot)$ defined in \eqref{eq::timestep}.
\end{definition}
The following lemma characterizes the model prediction errors defined in \eqref{eq::pf_def_est_err}.
\begin{lemma}
\label{lem:backdoor_UCB}
Let $\beta = CdH\sqrt{\log(d(T + nH)/\zeta)}$ and $\zeta \in (0, 1]$. Under Assumption \ref{asu::linMDP_backdoor}, it holds with probability at least $1 - 2\zeta$ that
\$
-2\Gamma^k_h(s_h, a_h) \leq \iota^k_h(s_h, a_h) \leq 0, \quad \forall (k, h) \in [K]\times[H], ~(s_h, a_h)\in\cS\times\cA.
\$
\end{lemma}
\begin{proof}
See \S\ref{pf:backdoor_UCB} for a detailed proof.
\end{proof}
In the sequel, we define the following operators,
\$
&(\JJ_h f)(s) = \langle f(s, \cdot), \pi^*_h(\cdot\given s) \rangle_\cA,\qquad(\JJ_{k, h} f)(s) = \langle f(s, \cdot), \pi^k_h(\cdot\given s) \rangle_\cA, \quad \forall s\in\cS.
\$
Meanwhile, recall that we define
\$
(\PP_h V)(s_h, a_h) &= \EE_{s_{h+1} \sim \PP(\cdot \given s_h, \doo(a_h))}\bigl[V(s_{h+1})\bigr], \quad \forall(s_h, a_h)\in\cS\times\cA.
\$
We define the following martingale adapted to the filtration $\{\cF_{k, h, m}\}_{(k, h, m)\in[K]\times[H]\times[2]}$,
\$
M_{k, h, m} = \sum_{\substack{(\tau, i, \ell) \in[K]\times[H]\times[2]\\ t(\tau, i, \ell) \leq t(k, h, m)} }D_{\tau, i, \ell},
\$
where
\$
D_{k, h, 1} &= \bigl(\JJ_{k, h}(Q^k_h - Q^{\pi^k, k}_h)\bigr)(s^k_h) - (Q^k_h - Q^{\pi^k, k}_h)\bigr)(s^k_h, a^k_h), \quad \forall (k, h)\in[K]\times[H],\\
D_{k, h, 2} &= \bigl(\PP_h(V^k_{h+1} - V^{\pi^k, k}_{h+1})\bigr)(s^k_h, a^k_h) - (V^k_{h+1} - V^{\pi^k, k}_{h+1})(s^k_{h+1}), \quad \forall (k, h)\in[K]\times[H].
\$
The following lemma is adapted from \cite{cai2019provably}.
\begin{lemma}[Lemma 4.2 of \cite{cai2019provably}]
\label{lem::4.2}
It holds that
\#\label{eq::lem_4.2}
\textrm{Regret}(T) &= \sum^K_{k = 1}V_1^{\pi^*}(x_1^k) - V_1^{\pi_k}(x_1^k)\notag\\
			    &= Y + \cM_{K, H, 2}+ \sum^K_{k = 1}\sum^H_{h = 1}\Bigl(\EE_{\pi^*}\bigl[\iota^k_h(s_h, a_h) \biggiven s_1 = s_1^k\bigr] - \iota^k_h(s^k_h, a^k_h)\Bigr),
\#
where
\#\label{eq::def_Y_backdoor}
Y = \sum^K_{k = 1}\sum^H_{h = 1} \EE_{\pi^*}\bigl[\langle Q^k_h(s_h, \cdot), \pi^*_h(\cdot\given s_h) - \pi^k_h(\cdot\given s_h) \rangle\biggiven s_1 = s_1^k\bigr] .
\#
\end{lemma}
\begin{proof}
See \cite{cai2019provably} for a detailed proof.
\end{proof}
In what follows, we upper bound the right-hand side of \eqref{eq::lem_4.2} in Lemma \ref{lem::4.2}. By Algorithm \ref{alg::backdoor}, it holds that $\pi^k_h$ is the greedy policy with respect to the action-value function $Q^k_h$. Hence, for $Y$ defined in \eqref{eq::def_Y_backdoor} of Lemma \ref{lem::4.2}, we have
\#\label{eq::regret_back_eq0}
Y =\sum^K_{k = 1}\sum^H_{h = 1} \EE_{\pi^*}\bigl[\langle Q^k_h(s_h, \cdot), \pi^*_h(\cdot\given s_h) - \pi^k_h(\cdot\given s_h) \rangle\biggiven s_1 = s_1^k\bigr]\leq 0.
\#
Meanwhile, following from the proof of Theorem 3.1 in \cite{cai2019provably}, it holds with probability at least $1 - \zeta/2$ that
\#\label{eq::regret_back_eq1}
M_{K, H, 2} \leq C_0\cdot\sqrt{d^3H^3T}\cdot\sqrt{\log(1/\zeta)},
\#
where $C_0>0$ is an absolute constant. In addition, following from Lemma \ref{lem:backdoor_UCB}, it holds with probability at least $1 - 2\zeta$ that
\#\label{eq::regret_back_eq2}
\sum^K_{k = 1}\sum^H_{h = 1}\Bigl(\EE_{\pi^*}\bigl[\iota^k_h(s_h, a_h) \biggiven s_1 = s_1^k\bigr] - \iota^k_h(s^k_h, a^k_h)\Bigr)\leq2\sum^K_{k = 1}\sum^H_{h = 1} \Gamma^k_h(s^k_h, a^k_h).
\#
Recall that for all $(s_h, a_h)\in\cS\times\cA$, we define
\#\label{eq::regret_back_eq3}
&\Gamma^k_h(s_h, a_h)= \beta\cdot\Bigl(\log \det\bigl(\Lambda^k_h + \psi_h(s_h, a_h)\psi_h(s_h, a_h)^\top\bigr) - \log\det(\Lambda^k_h)\Bigr)^{1/2}.
\#
Hence, by the Cauchy-Schwartz inequality, we obtain that
\#\label{eq::regret_back_eq4}
\sum^K_{k = 1}\sum^H_{h = 1} \Gamma^k_h(s^k_h, a^k_h) &= \beta\sum^K_{k = 1}\sum^H_{h = 1}\Bigl(\log \det\bigl(\Lambda^k_h + \psi_h(s^k_h, a^k_h)\psi_h(s^k_h, a^k_h)^\top\bigr) - \log\det(\Lambda^k_h)\Bigr)^{1/2}\notag\\
&\leq \beta\sum^H_{h = 1}\biggl(K\sum^K_{k = 1}\bigl(\log\det(\Lambda^{k+1}_h) - \log\det(\Lambda^k_h)\bigr)\biggr)^{1/2}\notag\\
&= \beta\sqrt{K}\sum^H_{h = 1}\bigl(\log\det(\Lambda^{K+1}_h) - \log\det(\Lambda^1_h)\bigr)^{1/2}.
\#
In what follows, we define 
\#\label{eq::regret_back_delta}
\Delta_H = \frac{1}{\sqrt{dH^2}}\sum^H_{h = 1}\bigl(\log\det(\Lambda^{K+1}_h) - \log\det(\Lambda^1_h)\bigr)^{1/2}.
\#
Thus, by plugging \eqref{eq::regret_back_delta} and $\beta = CdH\cdot\sqrt{\log(d(T + nH)/\zeta)}$ into \eqref{eq::regret_back_eq4}, it holds with probability at least $1 - 2\zeta$ that,
\#\label{eq::regret_back_eq5}
\sum^K_{k = 1}\sum^H_{h = 1} \Gamma^k_h(s^k_h, a^k_h) \leq C  \cdot\Delta_H\cdot \sqrt{d^3H^3T}\cdot\sqrt{\log\bigl(d(T + nH)/\zeta\bigr)},
\#
where recall that we define $T = HK$. By further plugging \eqref{eq::regret_back_eq5} into \eqref{eq::regret_back_eq2}, it holds with probability at least $1 - 2\zeta$ that,
\#\label{eq::regret_back_eq6}
&\sum^K_{k = 1}\sum^H_{h = 1}\Bigl(\EE_{\pi^*}\bigl[\iota^k_h(s_h, a_h) \biggiven s_1 = s_1^k\bigr] - \iota^k_h(s^k_h, a^k_h)\Bigr)\notag\\
&\qquad\leq 2C\cdot \Delta_H \cdot \sqrt{d^3H^3T}\cdot\sqrt{\log\bigl(d(T + nH)/\zeta\bigr)}.
\#
Finally, combining Lemma \ref{lem::4.2}, \eqref{eq::regret_back_eq0}, \eqref{eq::regret_back_eq1}, and \eqref{eq::regret_back_eq6}, it holds with probability at least $1 - 5\zeta/2$ that
\$
\textrm{Regret}(T) \leq C'\cdot \Delta_H\cdot \sqrt{d^3H^3 T}\cdot \sqrt{\log\bigl(d(T + nH)/\zeta\bigr)},
\$
where $C'>0$ is an absolute constant and
\$
\Delta_H = \frac{1}{\sqrt{dH^2}}\sum^H_{h = 1}\bigl(\log\det(\Lambda^{K+1}_h) - \log\det(\Lambda^1_h)\bigr)^{1/2}.
\$
Thus, we complete the proof of Theorem \ref{thm::regret_backdoor}.
\end{proof}

\subsection{Proof of Theorem \ref{thm::regret_frontdoor}}
\label{pf::regret_frontdoor}
\begin{proof}
In the sequel, we define the following operators,
\#\label{eq::def_J_frontdoor}
&(\JJ_h f)(s) = \langle f(s, \cdot), \pi^*_h(\cdot\given s) \rangle_\cA,\qquad(\JJ_{k, h} f)(s) = \langle f(s, \cdot), \pi^k_h(\cdot\given s) \rangle_\cA.
\#
Meanwhile, recall that we define the following transition operators,
\$
&\PP_{h+1/2} V(s_h, m_h) = \EE\Bigl[V(s_{h+1})~\Big|~ s_{h+1} \sim \PP\bigl(\cdot \biggiven s_h, \doo(m_h)\bigr)\Bigr],\quad\forall V:\cS\mapsto\RR,~(s_h, m_h)\in\cS\times\cM.\\
&\PP_{h} V'(s_h, a_h) = \EE\bigl[ V'(s_h, m_h) \biggiven m_h\sim \breve\cP_h(\cdot\given s, a)\bigr],\quad \forall V':\cS\times\cM\mapsto\RR, ~(s_h, a_h)\in\cS\times\cA.
\$
We further define for all $(k,h)\in[K]\times[H]$ the following transition operator,
\$
&\tilde\PP_{h+1/2}  V(s_h, a_h, m_h) = \EE\bigl[ V(s_{h+1}) ~\big|~ s_{h+1}\sim \PP(\cdot \given s_h, a_h, m_h)\bigr], \quad \forall V:\cS\mapsto\RR,~(s_h, a_h, m_h)\in\cS\times\cA\times\cM.
\$
We define the following model prediction errors,
\#\label{eq::pf_def_est_err_front}
&\iota^k_h(s_h, a_h) = -Q^k_h(s_h, a_h) + r_h(s_h, a_h) + (\PP_{h}V^{k}_{h+1/2})(s_h, a_h),\quad\forall(s_h, a_h)\in\cS\times\cA,\notag\\
&\iota^k_{h+1/2}(s_h, m_h) = -V^k_{h+1/2}(s_h, m_h) + (\PP_{h+1/2} V^k_{h+1})(s_h, m_h), \quad \forall (s_h, m_h)\in\cS\times\cM.
\#
In parallel to Definition \ref{def::filtration_backdoor}, we define the following filtrations that correspond to Algorithm \ref{alg::frontdoor}.
\begin{definition}[Filtration]
\label{def::filtration_front}
For $(k, h)\in[K]\times[H]$, we define $\cF'_{k, h, 1}$ the $\sigma$-algebra generated by the following set,
\#
B'_{k, h, 1} &= \bigl\{(s_{h}^i, a_{h}^i, m_h^i, r_h^i)\bigr\}_{(i, h) \in [n]\times [H]}\cup\bigl\{(s_j^\tau, a_j^\tau, m_j^\tau, r^\tau_j)\bigr\}_{(\tau, j)\in[k-1]\times[H]}\notag\\
&\qquad\cup\bigl\{(s_j^k, a_j^k, m_j^k, r_j^k)\bigr\}_{j\in[h-1]} \cup \bigl\{(s^k_h, a^k_h)\bigr\}.
\#
Similarly, we define $\cF'_{k, h, 2}$ the $\sigma$-algebra generated by the following set, 
\#
B'_{k, h, 2} = B'_{k, h, 1}\cup\{m^k_{h}\}\cup\{r^k_h\},
\#
and we define $\cF'_{k, h, 3}$ the $\sigma$-algebra generated by the following set,
\#
B'_{k, h, 3} = B'_{k, h, 2}\cup\{s^k_{h+1}\},
\#
Moreover, we define $\cF'_{0, h, 3}$ the $\sigma$-algebra generated by the set $\{(s_{h}^i, a_{h}^i, m_h^i, r_h^i)\}_{(i, h) \in [n]\times[H]}$ for all $h\in[H]$. We define the timestep index as follows,
\#\label{eq::timestep_front}
t'(k, h, m) = 3H\cdot k +3(h-1) + m.
\#
It then holds for $t'(k, h, m) \leq t'(k', h', m')$ that $\cF'_{k, h, m} \subseteq \cF'_{k', h', m'}$. Hence, the set of $\sigma$-algebra $\{\cF'_{k, h, m}\}_{(k, h, m)\in[K]\times[H]\times [3]}$ is a filtration with the timestep index $t'(\cdot, \cdot, \cdot)$ defined in \eqref{eq::timestep_front}.
\end{definition}
The following lemma characterizes the model prediction errors defined in \eqref{eq::pf_def_est_err_front}.
\begin{lemma}
\label{lem::front_door}
Let $\beta = CdH\sqrt{\log(d(T + nH)/\zeta)}$ and $\zeta \in (0, 1]$. Under Assumption \ref{asu::linMDP_frontdoor}, it holds with probability at least $1 - 4\zeta$ that
\#
&-2\Gamma^k_{h+1/2}(s_h, m_h) \leq \iota^k_{h+1/2}(s_h, m_h) \leq 0, ~ \forall (k, h) \in [K]\times[H], ~(s_h, m_h)\in\cS\times\cM,\\
&-2\Gamma^k_h(s_h, a_h) \leq \iota^k_h(s_h, a_h) \leq 0, \quad \forall (k, h) \in [K]\times[H], ~(s_h, a_h)\in\cS\times\cA.
\#
\end{lemma}
\begin{proof}
See \S\ref{sec::pf_frontdoor_UCB} for a detailed proof.
\end{proof}
Our goal is to upper bound the regret, which takes the following form,
\#\label{eq::frontdoor_regret_eq1}
\textrm{Regret}(T) &= \sum^K_{k = 1}V^{\pi^*}_1(s^k_1) - V^{\pi^k}_1(s^k_1)\notag\\
&= \underbrace{\sum^K_{k = 1}\bigl(V^{\pi^*}_1(s^k_1) - V^k_1(x^k_1)\bigr)}_{\textstyle{\textrm{(i)}}} + \underbrace{\sum^K_{k = 1}\bigl(V^{k}_1(s^k_1) - V^{\pi^k}_1(x^k_1)\bigr)}_{\textstyle{\textrm{(ii)}}},
\#
where $\{V^k_h\}_{(k, h)\in[K]\times[H]}$ is the output of Algorithm \ref{alg::frontdoor}. In what follows, we calculate terms (i) and (ii) on the right-hand side of \eqref{eq::frontdoor_regret_eq1} separately.
\vskip4pt
\noindent{\bf Term (i).} We now calculate term (i) on the right-hand side of \eqref{eq::frontdoor_regret_eq1}. By \eqref{eq::def_J_frontdoor}, for all $h\in[H]$, it holds that
\#\label{eq::frontdoor_i_eq1}
V^{\pi^*}_h - V^k_h &= \JJ_h Q^{\pi^*}_h + \JJ_{k, h} Q^k_h = \JJ_h(Q^{\pi^*}_h - Q^k_h) + (\JJ_h - \JJ_{k, h})Q^k_h.
\#
We first calculate the term $Q^{\pi^*}_h - Q^k_h$ on the right-hand side of \eqref{eq::frontdoor_i_eq1}. Recall that we define
\$
\iota^k_h = -Q^k_h + r_h + \PP_{h}V^{k}_{h+1/2},\qquad \iota^k_{h+1/2} = -V^k_{h+1/2} + \PP_{h+1/2} V^k_{h+1}.
\$
Meanwhile, following from the Bellman equation in \eqref{eq::frontdoor_bellman_0}, we obtain that
\$
Q^{\pi^*}_h = r_h + \PP_{h}V^{\pi^*}_{h+1/2},\qquad V^{\pi^*}_{h+1/2} = \PP_{h+1/2} V^{\pi^*}_{h+1}.
\$
Thus, it holds that
\#\label{eq::frontdoor_i_eq2}
Q^{\pi^*}_h - Q^k_h = \iota^k_h + \PP_h(V^{\pi^*}_{h+1/2} - V^{k}_{h+1/2}) = \iota^k_h + \PP_h\iota^k_{h+1/2} + \PP_{h}\PP_{h+1/2}(V^{\pi^*}_{h+1} - V^{k}_{h+1}).
\#
Recall that we set $V^{\pi^*}_{H+1} = V^{k}_{H+1} = 0$. Hence, upon recursion, we obtain from \eqref{eq::frontdoor_i_eq1} and \eqref{eq::frontdoor_i_eq2} that
\#\label{eq::frontdoor_i_eq3}
V^{\pi^*}_1 - V^k_1 &= \biggl(\prod^H_{h = 1}\JJ_h \PP_h\PP_{h+1/2}\biggr)(V^{\pi^*}_{H+1} - V^{k}_{H+1}) + \sum^H_{h = 1} \biggl(\prod^{h-1}_{i=1}\JJ_i \PP_i\PP_{i+1/2}\biggr)\JJ_h \iota^k_h\\
&\qquad + \sum^H_{h = 1} \biggl(\prod^{h-1}_{i=1}\JJ_i \PP_i\PP_{i+1/2}\biggr)\JJ_h \PP_h \iota^k_{h + 1/2} + \sum^H_{h = 1} \biggl(\prod^{h-1}_{i=1}\JJ_i \PP_i\PP_{i+1/2}\biggr)(\JJ_h - \JJ_{k, h})Q^k_h  \notag\\
&= \sum^H_{h = 1} \biggl(\prod^{h-1}_{i=1}\JJ_i \PP_i\PP_{i+1/2}\biggr) (\JJ_h \iota^k_h + \JJ_h \PP_h \iota^k_{h + 1/2})+ \sum^H_{h = 1} \biggl(\prod^{h-1}_{i=1}\JJ_i \PP_i\PP_{i+1/2}\biggr)(\JJ_h - \JJ_{k, h})Q^k_h .\notag
\#
By the definition of $\JJ_h$ and $\JJ_{k, h}$ in \eqref{eq::def_J_frontdoor}, we further obtain from \eqref{eq::frontdoor_i_eq3} that
\#\label{eq::frontdoor_i_eq4}
\sum^K_{k = 1}\bigl(V^{\pi^*}_1(s^k_1) - V^k_1(s^k_1)\bigr) &=  \sum^K_{k = 1}\sum^H_{h = 1} \EE_{\pi^*}\bigl[\iota^k_h(s_h, a_h) + \iota^k_{h + 1/2}(s_h, m_h) \biggiven s_1 = s^k_1\bigr]\\
&\qquad + \sum^K_{k = 1}\sum^H_{h = 1} \EE_{\pi^*}\bigl[\langle Q^k_h(s_h, \cdot), \pi^*_h(\cdot\given s_h) - \pi^k_h(\cdot\given s_h) \biggiven s_1 = s^k_1\bigr],\notag
\#
which completes the calculation of term (i) on the right-hand side of \eqref{eq::frontdoor_regret_eq1}.
\vskip4pt
\noindent{\bf Term (ii).} We now calculate term (ii) on the right-hand side of \eqref{eq::frontdoor_regret_eq1}. By \eqref{eq::def_J_frontdoor}, for all $h\in[H]$, we have
\#\label{eq::frontdoor_ii_eq1}
V^k_h(s^k_h) - V^{\pi^k}_h(s^k_h) = \bigl(\JJ_{k, h}(Q^k_h - Q^{\pi^k}_h)\bigr)(s^k_h).
\#
Meanwhile, by \eqref{eq::pf_def_est_err_front} it holds that
\#\label{eq::frontdoor_ii_eq2}
\iota^k_h(s^k_h, a^k_h) &= r_h(s^k_h, a^k_h) + (\PP_{h} V^k_{h + 1/2})(s^k_h, a^k_h) - Q^k_h(s^k_h, a^k_h) \notag\\
&= r_h(s^k_h, a^k_h)- Q^{\pi^k}_h(s^k_h, a^k_h) + \PP_{h} V^k_{h + 1/2}(s^k_h, a^k_h) + (Q^{\pi^k}_h - Q^k_h)(s^k_h, a^k_h)(s^k_h, a^k_h) \notag\\
&= \bigl(\PP_h(V^k_{h + 1/2} - V^{\pi^k}_{h+1/2})\bigr)(s^k_h, a^k_h) - (Q^k_h - Q^{\pi^k}_h )(s^k_h, a^k_h),
\#
where the second equality follows from the Bellman equation $Q^{\pi^k}_h(s_h, a_h) = r_h(s_h, a_h) + (\PP_h V^{\pi^k}_{h+1/2})(s_h, a_h)$. Similarly, we have
\#\label{eq::frontdoor_ii_eq3}
\iota^k_{h + 1/2}(s^k_h, m^k_h) &= \bigl(\PP_{h + 1/2}(V^k_{h + 1} - V^{\pi^k}_{h+1})\bigr)(s^k_h, m^k_h) - (V^k_{h + 1/2} - V^{\pi^k}_{h+1/2})(s^k_h, m^k_h).
\#
Thus, by combining \eqref{eq::frontdoor_ii_eq1}, \eqref{eq::frontdoor_ii_eq2}, and \eqref{eq::frontdoor_ii_eq3}, we have
\#\label{eq::frontdoor_ii_eq4}
&(V^k_h - V^{\pi^k}_h)(s^k_h) + \iota^k_h(s^k_h, a^k_h) + \iota^k_{h + 1/2}(s^k_h, m^k_h)  \notag\\
&\qquad =(V^k_{h + 1} - V^{\pi^k}_{h+1})(s^k_{h+1}) + \underbrace{\bigl(\JJ_{k, h}(Q^k_h - Q^{\pi^k}_h)\bigr)(s^k_h) - (Q^k_h - Q^{\pi^k}_h )(s^k_h, a^k_h)}_{\textstyle{D_{k, h, 1}}}\\
&\qquad\qquad+ \underbrace{\bigl(\PP_h(V^k_{h + 1/2} - V^{\pi^k}_{h+1/2})\bigr)(s^k_h, a^k_h) - (V^k_{h + 1/2} - V^{\pi^k}_{h+1/2})(s^k_h, m^k_h)}_{\textstyle{D_{k, h, 2}}}\notag\\
&\qquad \qquad + \underbrace{\bigl(\PP_{h + 1/2}(V^k_{h + 1} - V^{\pi^k}_{h+1})\bigr)(s^k_h, m^k_h)- (V^k_{h + 1} - V^{\pi^k}_{h+1})(s^k_{h+1})}_{\textstyle{D_{k, h, 3}}} .\notag
\#
Meanwhile, note that $V^{\pi^k}_{H+1} = V^k_{H + 1} = 0$. Hence, by recursively applying \eqref{eq::frontdoor_ii_eq4}, we obtain that 
\#\label{eq::frontdoor_ii_eq5}
&(V^k_1 - V^{\pi^k}_1)(s^k_1) \notag\\
&\quad= \sum^H_{h = 1}(D_{k, h, 1}+ D_{k, h, 2} + D_{k, h, 3})-\sum^H_{h = 1}\bigl(\iota^k_h(s^k_h, a^k_h) + \iota^k_{h + 1/2}(s^k_h, m^k_h)\bigr).
\#
By the definition of filtration in \eqref{def::filtration_front},  for the terms $D_{k, h, 1}$, $D_{k, h, 2}$ and $D_{k, h, 3}$ on the right-hand side of \eqref{eq::frontdoor_ii_eq4}, it holds for all $(k, h)\in[K]\times[H]$ that
\$
D_{k, h, 1}\in \cF_{k, h, 1}, \quad D_{k, h, 2}\in \cF_{k, h, 2}, \quad D_{k, h, 3}\in \cF_{k, h,3}. 
\$
Moreover, it holds that
\$
\EE[D_{k, h, 1} \given \cF_{k, h-1, 3}] = \EE[D_{k, h, 2} \given \cF_{k, h, 1}] =  \EE[D_{k, h, 3} \given \cF_{k, h, 2}] = 0. 
\$
Hence, the terms $D_{k, h, 1}$, $D_{k, h, 2}$ and $D_{k, h, 3}$ defines a martingale $M'_{k, h, m}$ with respect to the timestep index $t'(\cdot, \cdot, \cdot)$ as follows,
\#\label{eq::frontdoor_martingale}
M'_{k, h, m} = \sum_{\substack{(\tau, i, \ell) \in[K]\times[H]\times[3]\\ t'(\tau, i, \ell) \leq t'(k, h, m)} }D_{\tau, i, \ell},
\#
where $t'(\cdot, \cdot, \cdot)$ is defined in \eqref{eq::timestep_front} of Definition \ref{def::filtration_front}. In specific, we have
\#\label{eq::frontdoor_ii_eq6}
M'_{K, H, 3} = \sum^K_{k = 1}\sum^H_{h = 1}(D_{k, h, 1} + D_{k, h, 2} + D_{k, h, 3}).
\#
By further taking sum of \eqref{eq::frontdoor_ii_eq5} over $k \in [K]$, we obtain from \eqref{eq::frontdoor_ii_eq6} that
\#\label{eq::frontdoor_ii_eq7}
\sum^K_{k = 1}(V^k_1 - V^{\pi^k}_1)(s^k_1) &=M'_{K, H, 3}-\sum^K_{k = 1}\sum^H_{h = 1}\bigl(\iota^k_h(s^k_h, a^k_h) + \iota^k_{h + 1/2}(s^k_h, m^k_h)\bigr),
\#
which completes the calculation of term (ii) on the right-hand side of \eqref{eq::frontdoor_regret_eq1}.

Finally, by plugging \eqref{eq::frontdoor_i_eq4} and \eqref{eq::frontdoor_ii_eq7} into \eqref{eq::frontdoor_regret_eq1}, we conclude that
\#\label{eq::frontdoor_regret_eq2}
\textrm{Regret}(T)  &= \sum^K_{k = 1}\sum^H_{h = 1} \EE_{\pi^*}\bigl[\langle Q^k_h(s_h, \cdot), \pi^*_h(\cdot\given s_h) - \pi^k_h(\cdot\given s_h) \biggiven s_1 = s^k_1\bigr]+ M'_{K, H, 3}\\
&\qquad + \sum^K_{k = 1}\sum^H_{h = 1} \EE_{\pi^*}\bigl[\iota^k_h(s_h, a_h) + \iota^k_{h + 1/2}(s_h, m_h) \biggiven s_1 = s^k_1\bigr]\notag\\
&\qquad - \sum^K_{k = 1}\sum^H_{h = 1}\bigl(\iota^k_h(s^k_h, a^k_h) + \iota^k_{h + 1/2}(s^k_h, m^k_h)\bigr) ,\notag
\#
where $M'_{K, H, 3}$ is defined in \eqref{eq::frontdoor_ii_eq6}.

We now upper bound the right-hand side of \eqref{eq::frontdoor_regret_eq2}. The following proof is similar to that of Theorem \ref{thm::regret_backdoor} in \S\ref{pf::regret_backdoor}. In the sequel, we define
\$
Y' &= \sum^K_{k = 1}\sum^H_{h = 1} \EE_{\pi^*}\bigl[\langle Q^k_h(s_h, \cdot), \pi^*_h(\cdot\given s_h) - \pi^k_h(\cdot\given s_h) \biggiven s_1 = s^k_1\bigr] ,\\
Z' &= \sum^K_{k = 1}\sum^H_{h = 1} \EE_{\pi^*}\bigl[\iota^k_h(s_h, a_h) + \iota^k_{h + 1/2}(s_h, m_h) \biggiven s_1 = s^k_1\bigr] - \sum^K_{k = 1}\sum^H_{h = 1}\bigl(\iota^k_h(s^k_h, a^k_h) + \iota^k_{h + 1/2}(s^k_h, m^k_h)\bigr).
\$
It then follows from \eqref{eq::frontdoor_regret_eq2} that
\#\label{eq::frontdoor_regret_eq3}
\textrm{Regret}(T) = Y' + M'_{K, H, 3}+ Z'.
\#
Recall that we set $\pi^k_h$ to be the greedy policy with respect to the action-value function $Q^k_h$. Thus, it holds that 
\#\label{eq::frontdoor_Y}
Y' =\sum^K_{k = 1}\sum^H_{h = 1} \EE_{\pi^*}\bigl[\langle Q^k_h(s_h, \cdot), \pi^*_h(\cdot\given s_h) - \pi^k_h(\cdot\given s_h) \biggiven s_1 = s^k_1\bigr]\leq 0.
\#
Meanwhile, following from the truncation of $Q^k_h$ in Algorithm \ref{alg::frontdoor} and the assumption that $r_h \in [0, 1]$, for terms $D_{k, h, i}$ defined in \eqref{eq::frontdoor_ii_eq4}, we have
\$
|D_{k, h, i}| \leq 2H, \quad \forall(k, h, i)\in[K]\times[H]\times[3].
\$
Hence, by the Azumas-Hoeffding lemma, it holds with probability at least $1-\zeta$ that
\#\label{eq::frontdoor_regret_eq4}
M'_{K,H,3} \leq C_1 \cdot \sqrt{d^3H^3T}\cdot \sqrt{\log(dT/\zeta)},
\#
where $M'_{K,H,3}$ is the martingale defined in \eqref{eq::frontdoor_martingale}, $C_1>0$ is an absolute constant, and $T = HK$. Following from Lemma \ref{lem::front_door}, it holds with probability at least $1 - 4\zeta$ that
\#\label{eq::frontdoor_regret_eq5}
Z' \leq 2\sum^K_{k = 1}\sum^H_{h = 1}\Gamma^k_{h+1/2}(s^k_h, m^k_h) +2\sum^K_{k = 1}\sum^H_{h = 1} \Gamma^k_{h}(s^k_h, a^k_h).
\#
Following from the definition of $\Gamma^k_{h+1/2}$ in \eqref{eq::UCB_frontdoor}, we obtain that
\#\label{eq::frontdoor_regret_eq6}
\sum^K_{k = 1}\sum^H_{h = 1}\Gamma^k_{h+1/2}(s^k_h, m^k_h) &= 2\beta\sum^K_{k = 1}\sum^H_{h = 1}\Bigl(\log \det\bigl(\Lambda^k_{1, h} + \psi_h(s^k_h, m^k_h)\psi_h(s_h, m_h)^\top\bigr) - \log\det(\Lambda^k_{1, h})\Bigr)^{1/2}\notag\\
&=2\beta\sum^K_{k = 1}\sum^H_{h = 1}\bigl(\log\det(\Lambda^{k+1}_{1, h}) - \log\det(\Lambda^k_{1, h})\bigr)^{1/2}.
\#
Thus, by the Cauchy-Schwartz inequality, we obtain from \eqref{eq::frontdoor_regret_eq6} that
\#\label{eq::frontdoor_regret_eq7}
\sum^K_{k = 1}\sum^H_{h = 1}\Gamma^k_{h+1/2}(s^k_h, m^k_h) &\leq \beta \sum^H_{h = 1}\Biggl(K \cdot \sum^K_{k = 1}\bigl(\log\det(\Lambda^{k+1}_{1, h}) - \log\det(\Lambda^1_{1, h})\bigr)\Biggr)^{1/2}\notag\\
&\leq \beta\cdot\sqrt{K} \sum^H_{h = 1}\bigl(\log\det(\Lambda^{K+1}_{1, h}) - \log\det(\Lambda^1_{1, h})\bigr)^{1/2}.
\#
Similarly, we obtain that
\#\label{eq::frontdoor_regret_eq8}
\sum^K_{k = 1}\sum^H_{h = 1}\Gamma^k_{h}(s^k_h, a^k_h) \leq \beta\cdot\sqrt{K} \sum^H_{h = 1}\bigl(\log\det(\Lambda^{k+1}_{2, h}) - \log\det(\Lambda^1_{2, h})\bigr)^{1/2}.
\#
In what follows, we define
\$
\Delta_{1, H} =\frac{1}{\sqrt{dH^2}}\sum^H_{h = 1}\bigl(\log\det(\Lambda^{K+1}_{1, h}) - \log\det(\Lambda^1_{1, h})\bigr)^{1/2},\\
\Delta_{2, H} =\frac{1}{\sqrt{dH^2}}\sum^H_{h = 1}\bigl(\log\det(\Lambda^{k+1}_{2, h}) - \log\det(\Lambda^1_{2, h})\bigr)^{1/2}.
\$
By plugging \eqref{eq::frontdoor_regret_eq7}, \eqref{eq::frontdoor_regret_eq8}, and $\beta = CdH\cdot\sqrt{\log(d(T + nH)/\zeta)}$ into \eqref{eq::frontdoor_regret_eq5}, we obtain that 
\#\label{eq::frontdoor_regret_eq9}
Z' \leq 2C\cdot(\Delta_{1, H} + \Delta_{2, H})\cdot\sqrt{d^3H^3T}\cdot \sqrt{\log\bigl(d(T + nH)/\zeta\bigr)},
\#
which holds with probability at least $1 - 4\zeta$. Here recall that we define $T = HK$. Finally, by plugging \eqref{eq::frontdoor_Y}, \eqref{eq::frontdoor_regret_eq4}, and \eqref{eq::frontdoor_regret_eq9} into \eqref{eq::frontdoor_regret_eq3}, it holds with probability at least $1 - 5\zeta$ that
\$
\textrm{Regret}(T) \leq C' \cdot(\Delta_{1, H} + \Delta_{2, H})\cdot\sqrt{d^3H^3T}\cdot \sqrt{\log\bigl(d(T + nH)/\zeta\bigr)},
\$
where $C'>0$ is an absolute constant. Thus, we complete the proof of Theorem \ref{thm::regret_frontdoor}.
\end{proof}

\section{Proof of Auxiliary Result}

\subsection{Proof of Lemma \ref{lem:backdoor_UCB}}
\label{pf:backdoor_UCB}
\begin{proof}
Recall that we define
\$
(\PP_h V)(s_h, a_h) &= \EE\Bigl[V(s_{h+1})\, \Big|\, s_{h+1} \sim \PP\bigl(\cdot \biggiven s_h, \doo(a_h)\bigr)\Bigr]\notag\\
      		      &= \EE\bigl[V(s_{h+1})\biggiven s_{h+1} \sim \cP_h(\cdot \given s_h, a_h, u_h), u_h\sim\tilde\cP_h(\cdot\given s_h)\bigr],
\$
where the second equality follows from Proposition \ref{lem::backdoor}. In the sequel, we define
\$
(\tilde\PP_h V)(s_h, a_h, u_h) &= \EE\Bigl[V(s_{h+1}) \,\Big|\, s_{h+1} \sim \cP_h\bigl(\cdot \biggiven s_h, a_h, u_h\bigr)\Bigr].
\$
By Assumption \ref{asu::linMDP_backdoor}, we obtain that
\#\label{eq::backdoor_UCB_1}
&\PP_h V^k_{h+1} = \psi_h^\top\langle \mu_h, V^k_{h+1} \rangle = \psi_h^\top (\Lambda^k_h)^{-1}\Lambda^k_h \langle \mu_h, V^k_{h+1} \rangle,\quad\tilde\PP_h V^k_{h+1} = \phi_h^\top\langle \mu_h, V^k_{h+1} \rangle.
\#
Recall that 
\$
\Lambda^k_h = \sum^{k-1}_{\tau = 1} \psi_h(s^\tau_h, a^\tau_h) \psi_h(s^\tau_h, a^\tau_h)^\top+ \sum^n_{i = 1} \phi_h(s_h^i, a_h^i, u_h^i)\phi_h(s_h^i, a_h^i, u_h^i)^\top + \lambda I.
\$
Therefore, by \eqref{eq::backdoor_UCB_1}, we obtain that
\#\label{eq::backdoor_UCB_2}
(\PP_h V^k_{h+1}) (\cdot, \cdot) &= \psi_h(\cdot, \cdot)^\top (\Lambda^k_h)^{-1} \biggl(\sum^{k-1}_{\tau = 1} \psi_h(s^\tau_h, a^\tau_h) \psi_h(s^\tau_h, a^\tau_h)^\top \langle \mu_h, V^k_{h+1} \rangle + \lambda \cdot\langle \mu_h, V^k_{h+1} \rangle \notag\\
&\qquad\qquad\qquad\qquad\quad+  \sum^n_{i = 1} \phi_h(s_h^i, a_h^i, u_h^i)\phi_h(s_h^i, a_h^i, u_h^i)^\top \langle \mu_h, V^k_{h+1} \rangle \biggr)\notag\\
&= \psi_h(\cdot, \cdot)^\top (\Lambda^k_h)^{-1} \biggl( \sum^{k-1}_{\tau = 1} \psi_h(s^\tau_h, a^\tau_h)\cdot (\PP_h V^k_{h+1})(s^\tau_h, a^\tau_h) +  \lambda \cdot\langle \mu_h, V^k_{h+1} \rangle\\
&\qquad\qquad\qquad\qquad\quad+ \sum^n_{i = 1} \phi_h(s_h^i, a_h^i, u_h^i)\cdot (\tilde\PP_h V^k_{h+1})(s_h^i, a_h^i, u_h^i) \biggr).\notag
\#
Recall that we define the counterfactual reward as follows,
\#\label{eq::backdoor_UCB_R}
R_h(s_h, a_h) = \EE_{u_h}\bigl[r(s_h, a_h, u_h) \biggiven S_h = s_h\bigr], \quad \forall (s_h, a_h)\in\cS\times\cA.
\#
It then follows from Assumption \ref{asu::linMDP_backdoor} and Proposition \ref{prop::backdoor_feature} that $R_h(\cdot, \cdot) = \psi_h(\cdot, \cdot)^\top\theta_h$. Hence, it holds for all $h\in [H]$ that
\#\label{eq::backdoor_UCB_2.5}
r_h(\cdot, \cdot, \cdot) &= \phi_h(\cdot, \cdot, \cdot)^\top \theta_h = \phi_h(\cdot, \cdot, \cdot)^\top (\Lambda^k_h)^{-1} \Lambda^k_h \theta_h \notag\\
&= \phi_h(\cdot, \cdot, \cdot)^\top (\Lambda^k_h)^{-1}\biggl(\sum^{k-1}_{\tau = 1} \psi_h(s^\tau_h, a^\tau_h) \psi_h(s^\tau_h, a^\tau_h)^\top\theta_h + \lambda \cdot\langle \mu_h, V^k_{h+1} \rangle \notag\\
&\qquad\qquad\qquad\qquad\quad+  \sum^n_{i = 1} \phi_h(s_h^i, a_h^i, u_h^i)\phi_h(s_h^i, a_h^i, u_h^i)^\top \theta_h \biggr)\notag\\
&= \phi_h(\cdot, \cdot, \cdot)^\top (\Lambda^k_h)^{-1}\biggl(\sum^{k-1}_{\tau = 1} \psi_h(s^\tau_h, a^\tau_h) \cdot R_h(s^\tau_h, a^\tau_h)+ \lambda \cdot \theta_h \notag\\
&\qquad\qquad\qquad\qquad\quad+  \sum^n_{i = 1} \phi_h(s_h^i, a_h^i, u_h^i)\cdot \EE[r_h\given s^i_h, a^i_h, u^i_h]  \biggr).
\#
Meanwhile, following from the explicit update of $\omega^k_h$ in \eqref{eq::backdoor_update}, we obtain that 
\#\label{eq::backdoor_UCB_3}
\psi_h(\cdot, \cdot)^\top \omega^k_h&= \psi_h(\cdot, \cdot)^\top (\Lambda^k_h)^{-1}  \biggl( \sum^{k-1}_{\tau = 1} \psi_h(s^\tau_h, a^\tau_h) \cdot \bigl(V^k_{h+1}(s^\tau_{h+1}) + r^\tau_h \bigr)\\
&\qquad\qquad\qquad\quad\qquad+ \sum^n_{i = 1} \phi_h(s_h^i, a_h^i, u_h^i) \cdot \bigl(V^k_{h+1}(s^i_{h+1}) + r^i_h\bigr)\biggr).\notag
\#

Hence, combining \eqref{eq::backdoor_UCB_2}, \eqref{eq::backdoor_UCB_2.5}, and \eqref{eq::backdoor_UCB_3}, we obtain that
\#\label{eq::backdoor_UCB_4}
&\psi_h(\cdot, \cdot)^\top \omega^k_h - R_h(\cdot, \cdot)-(\PP_h V^k_{h+1})(\cdot, \cdot) \notag\\
&\qquad= \psi_h(\cdot, \cdot)^\top (\Lambda^k_h)^{-1} (S_{1, h} + S_{2, h} + S_{3, h} + S_{4, h}) - \psi_h(\cdot, \cdot)^\top \lambda \cdot\bigl(\langle \mu_h, V^k_{h+1} \rangle + \theta_h\bigr),
\#
where we define
\#\label{eq::backdoor_UCB_SS}
S_{1, h} &=  \sum^{k-1}_{\tau = 1} \psi_h(s^\tau_h, a^\tau_h) \cdot \bigl(V^k_{h+1}(s^\tau_{h+1}) - (\PP_h V^k_{h+1})(s^\tau_h, a^\tau_h)  \bigr),\\
S_{2, h} &= \sum^n_{i = 1} \phi_h(s_h^i, a_h^i, u_h^i) \cdot \bigl(V^k_{h+1}(s^i_{h+1})-(\tilde\PP_h V^k_{h+1})(s_h^i, a_h^i, u_h^i) \bigr),\notag\\
S_{3, h} &=  \sum^{k-1}_{\tau = 1} \psi_h(s^\tau_h, a^\tau_h) \cdot \bigl( r^\tau_h -  R(s^\tau_h, a^\tau_h) \bigr),\quad \textrm{and}\quad S_{4, h} =\sum^n_{i = 1} \phi_h(s_h^i, a_h^i, u_h^i) \cdot \bigl( r^i_h -  \EE[r_h\given s^i_h, a^i_h, u^i_h] \bigr).\notag
\#

In what follows, we upper bound the right-hand side of \eqref{eq::backdoor_UCB_4}. By the Cauchy-Schwartz inequality, we obtain that
\#\label{eq::backdoor_UCB_5}
&|\psi_h(\cdot, \cdot)^\top \omega^k_h -R_h(\cdot, \cdot)- (\PP_h V^k_{h+1})(\cdot, \cdot)| \\
&\quad\leq \bigl(\psi_h(\cdot, \cdot)^\top (\Lambda^k_h)^{-1} \psi_h(\cdot, \cdot)\bigr)^{1/2}\cdot \biggl( \biggl\|\sum^4_{\ell = 1}S_{\ell, h}\biggr\|_{(\Lambda^k_h)^{-1}}  + \lambda \cdot\bigl( \|\langle \mu_h, V^k_{h+1} \rangle\|_{(\Lambda^k_h)^{-1}} + \|\theta_h\|_{(\Lambda^k_h)^{-1}}  \bigr)\biggr),\notag
\#
where $S_{1, h}$, $S_{2, h}$, $S_{3, h}$, and $S_{4, h}$ are defined in \eqref{eq::backdoor_UCB_SS}. By Lemma \ref{lem::self_norm_process}, for $\lambda = 1$, it holds with probability at least $1-2\zeta$ that
\#\label{eq::backdoor_UCB_6}
 \biggl\|\sum^4_{\ell = 1}S_{\ell, h}\biggr\|_{(\Lambda^k_h)^{-1}}\leq C'dH\sqrt{\log\bigl(2(C+1)d(T + nH)/\zeta\bigr)},
\#
where $C>0$ and $C'>0$ are absolute constants. Meanwhile, by Assumption \ref{asu::linMDP_backdoor}, it holds that
\#\label{eq::backdoor_UCB_7}
\|\langle \mu_h, V^k_{h+1} \rangle\|_{(\Lambda^k_h)^{-1}} &\leq \|\langle \mu_h, V^k_{h+1} \rangle\|_2/\sqrt{\lambda}\notag\\
&\leq \biggl(\sum^d_{\ell = 1}\|\mu_{\ell, h}\|^2_1 \biggr)^{1/2}\cdot \|V_{k+1}^h\|_{\infty}/\sqrt{\lambda}\leq H\sqrt{d/\lambda},
\#
where the first inequality follows from the fact that $\Lambda^k_h \succeq \lambda I$, the second inequality follows from the H\"older's inequality, and the third inequality follows from Assumption \ref{asu::linMDP_backdoor} and the fact that $V^h_{k+1} \leq H$. Similarly, it holds from Assumption \ref{asu::linMDP_backdoor} that
\#\label{eq::backdoor_UCB_7.5}
\|\theta_h\|_{(\Lambda^k_h)^{-1}} \leq \|\theta_h\|_2/\sqrt{\lambda} \leq \sqrt{d/\lambda}.
\#
Finally, by plugging \eqref{eq::backdoor_UCB_6}, \eqref{eq::backdoor_UCB_7}, and \eqref{eq::backdoor_UCB_7.5} into \eqref{eq::backdoor_UCB_5} with $\lambda =1$, it holds with probability at least $1 - 2\zeta$ that
\#\label{eq::backdoor_UCB_8}
|\psi_h(\cdot, \cdot)^\top \omega^k_h - R_h(\cdot, \cdot) - (\PP_h V^k_{h+1})(\cdot, \cdot)| \leq \beta/\sqrt{2}\cdot \bigl(\psi_h(\cdot, \cdot)^\top (\Lambda^k_h)^{-1} \psi_h(\cdot, \cdot)\bigr)^{1/2},
\#
where we set $\beta = C''dH\sqrt{\log(d(T + nH)/\zeta)}$ for a sufficiently large absolute constant $C''>0$. By further applying Lemma \ref{lem::Gamma} to \eqref{eq::backdoor_UCB_8}, for $\lambda =1$, it holds with probability at least $1 - 2\zeta$ that
\#\label{eq::backdoor_UCB_8.5}
&|\psi_h(\cdot, \cdot)^\top \omega^k_h -R_h(\cdot, \cdot) - (\PP_h V^k_{h+1})(\cdot, \cdot)| \notag\\
&\quad\leq \beta\cdot\Bigl(\log \det\bigl(\Lambda^k_h + \psi_h(\cdot, \cdot)\psi_h(\cdot, \cdot)^\top\bigr) - \log\det(\Lambda^k_h)\Bigr)^{1/2}=\Gamma^k_h(\cdot, \cdot). 
\#
Recall that we set
\$
Q^k_h(\cdot, \cdot) = \min\bigl\{\psi_h(\cdot, \cdot)^\top \omega^k_h + \Gamma^k_h(\cdot, \cdot), H-h\bigr\}.
\$
Hence, by \eqref{eq::backdoor_UCB_8.5}, it holds with probability at least $1 - 2\zeta$ that
\$
-\iota^k_h(\cdot, \cdot) &= Q^k_h(\cdot, \cdot) - R_h(\cdot, \cdot) -( \PP_h V^k_{h+1})(\cdot, \cdot)\notag\\
&\leq \psi_h(\cdot, \cdot)^\top\omega_h^k + \Gamma^k_h(\cdot, \cdot)- R_h(\cdot, \cdot) - (\PP_h V^k_{h+1})(\cdot, \cdot) \leq 2\Gamma^k_h(\cdot, \cdot),
\$
and
\$
\iota^k_h(\cdot, \cdot) &= -Q^k_h(\cdot, \cdot) +  R_h(\cdot, \cdot) + (\PP_h V^k_{h+1})(\cdot, \cdot)\notag\\
&\leq \max\bigl\{(\PP_h V^k_{h+1})(\cdot, \cdot) +  R_h(\cdot, \cdot) - \psi_h(\cdot, \cdot)^\top\omega_h^k - \Gamma^k_h, R_h(\cdot, \cdot) + (\PP_h V^k_{h+1})(\cdot, \cdot) - H + h\bigr\} \leq 0,
\$
where the second inequality follows from \eqref{eq::backdoor_UCB_8.5} the facts that $V^k_{h+1} \leq H-h-1$ and $R_h\leq 1$. In conclusion, it holds with probability at least $1 - 2\zeta$ that
\$
-2\Gamma^k_h(\cdot, \cdot) \leq \iota^k_h(\cdot, \cdot) \leq 0,
\$
which concludes the proof of Lemma \ref{lem:backdoor_UCB}.
\end{proof}


\subsection{Proof of Lemma \ref{lem::front_door}}
\label{sec::pf_frontdoor_UCB}
\begin{proof}
Recall that we define the following transition operators,
\#
&\PP_{h+1/2} V(s_h, m_h) = \EE\Bigl[V(s_{h+1})\,\Big|\, s_{h+1} \sim \PP\bigl(\cdot \biggiven s_h, \doo(m_h)\bigr)\Bigr]\notag\\
&\tilde\PP_{h+1/2} V(s_h, a_h, m_h) = \EE\bigl[V(s_{h+1}) \,\big|\, s_{h+1}\sim \PP(\cdot \given s_h, a_h, m_h)\bigr].
\#
Following from Assumption \ref{asu::linMDP_frontdoor} and \eqref{eq::front_back_adjust}, we have
\#
\label{eq::frontUCB_eq1}\PP_{h+1/2} V^k_{h+1} &= \psi^\top_h\langle\mu_h, V^k_{h+1}\rangle =  \psi^\top_h(\Lambda^k_{1, h})^{-1}\Lambda^k_{1, h}\langle\mu_h, V^k_{h+1}\rangle,\\
\label{eq::frontUCB_eq2}\tilde\PP_{h+1/2} V^k_{h+1} &= \phi^\top_h\langle\mu_h, V^k_{h+1}\rangle,
\#
where we define
\#\label{eq::frontUCB_eq3}
\Lambda^k_{1, h} = \sum^{k-1}_{\tau = 1} \psi_h(s^\tau_h, m^\tau_h) \psi(s^\tau_h, m^\tau_h)^\top+ \sum^n_{i = 1} \phi_h(s_h^i, a_h^i, m_h^i)\phi_h(s_h^i, a_h^i, m_h^i)^\top + \lambda I.
\#
Hence, following from \eqref{eq::frontUCB_eq1}, it holds for all $(s_h, m_h)\in\cS\times\cM$ that 
\#\label{eq::frontUCB_eq4}
&\PP_{h+1/2} V^k_{h+1}(s_h, m_h)\notag\\
&\quad= \psi_h(s_h, m_h)^\top(\Lambda^k_{1, h})^{-1}\biggl(\sum^{k-1}_{\tau = 1}\psi_h(s^\tau_h, m^\tau_h) \psi(s^\tau_h, m^\tau_h)^\top\langle\mu_h, V^k_{h+1}\rangle + \lambda\cdot \langle\mu_h, V^k_{h+1}\rangle\\
&\quad\qquad\qquad\qquad\qquad\qquad\quad + \sum^n_{i = 1}\phi_h(s^i_h, a^i_h, m^i_h)\phi_h(s^i_h, a^i_h, m^i_h)^\top\langle\mu_h, V^k_{h+1}\rangle  \biggr).\notag
\#
By plugging \eqref{eq::frontUCB_eq1} and \eqref{eq::frontUCB_eq2} into \eqref{eq::frontUCB_eq4}, we further obtain that
\#\label{eq::frontUCB_eq5}
&\PP_{h+1/2} V^k_{h+1}(s_h, m_h)\notag\\
&\quad=\psi_h(s_h, m_h)^\top(\Lambda^k_{1, h})^{-1}\biggl(\sum^{k-1}_{\tau = 1}\psi_h(s^\tau_h, m^\tau_h)\cdot (\PP_{h+1/2} V^k_{h+1})(s^\tau_h, m^\tau_h) + \lambda\cdot \langle\mu_h, V^k_{h+1}\rangle\\
&\qquad\qquad\qquad\qquad\qquad\qquad + \sum^n_{i = 1}\phi_h(s^i_h, a^i_h, m^i_h)\cdot(\tilde\PP_{h+1/2} V^k_{h+1})(s^i_h, a^i_h, m^i_h)\biggr).\notag
\#
Following from the update of $\omega^k_{1, h}$ in \eqref{eq::frontdoor_update_1}, it holds for all $h\in[H]$ and $(s_h, m_h)\in\cS\times\cM$ that
\#\label{eq::frontUCB_eq6}
\psi_h(s_h, m_h)^\top \omega^k_{1, h} = \psi_h(s_h, m_h)^\top(\Lambda^{k}_{1, h})^{-1}  \biggl(& \sum^{k-1}_{\tau = 1} \psi_h(s^\tau_h, m^\tau_h) \cdot V^k_{h+1}(s^\tau_{h+1}) \\
&+ \sum^n_{i = 1} \phi_h(s_h^i, a_h^i, m_h^i)\cdot V^k_{h+1}(s^i_{h+1}) \biggr).\notag
\#
Hence, combining \eqref{eq::frontUCB_eq5} and \eqref{eq::frontUCB_eq6}, we obtain for all $h\in[H]$ and $(s_h, m_h)\in\cS\times\cM$ that
\#\label{eq::frontUCB_eq7}
&\psi_h(s_h, m_h)^\top \omega^k_{1, h} - \PP_{h+1/2} V^k_{h+1}(s_h, m_h) \notag\\
&\qquad=  \psi_h(s_h, m_h)^\top(\Lambda^{k}_{1, h})^{-1}(S'_{1, h} + S'_{2, h}) + \lambda\cdot \psi_h(s, m)^\top\langle\mu_h, V^k_{h+1}\rangle,
\#
where we define
\$
S'_{1, h} &= \sum^{k-1}_{\tau = 1}\psi_h(s^\tau_h, m^\tau_h)\cdot \bigl( V^k_{h+1}(s^\tau_{h+1})- (\PP_{h+1/2} V^k_{h+1})(s^\tau_h, m^\tau_h)\bigr),\notag\\
S'_{2, h} &= \phi_h(s^i_h, a^i_h, m^i_h)\cdot\bigl(V^k_{h+1}(s^i_{h+1}) - (\tilde\PP_{h+1/2} V^k_{h+1})(s^i_h, a^i_h, m^i_h)\bigr).
\$
We now upper bound the right-hand side of \eqref{eq::frontUCB_eq7}. By the Cauchy-Schwartz inequality, we obtain from \eqref{eq::frontUCB_eq7} that
\#\label{eq::frontUCB_eq9}
&|\psi_h^\top \omega^k_{1, h} - \PP_{h+1/2} V^k_{h+1}| \notag\\
&\qquad\leq \bigl(\psi_h^\top(\Lambda^{k}_{1, h})^{-1} \psi_h\bigr)^{1/2}\cdot\bigl(\|S'_{1, h}+S'_{2, h}\|_{(\Lambda^k_h)^{-1}} + \lambda \cdot \|\langle \mu_h, V^k_{h+1} \rangle\|_{(\Lambda^k_h)^{-1}}\bigr).
\#
Following from similar analysis to the proof of Lemma \ref{lem::self_norm_process} in \S\ref{sec::aux_lemma}, for $\lambda = 1$, it holds with probability at least $1-2\zeta$ that
\#\label{eq::frontUCB_eq10}
\|S'_{1, h} + S'_{2, h}\|_{(\Lambda^k_h)^{-1} }\leq C'dH\sqrt{\log\bigl(2(C+1)d(T + nH)/\zeta\bigr)}.
\#
Meanwhile, by Assumption \ref{asu::linMDP_frontdoor}, we have
\#\label{eq::frontUCB_eq11}
\|\langle \mu_h, V^k_{h+1} \rangle\|_{(\Lambda^k_h)^{-1}} &\leq \|\langle \mu_h, V^k_{h+1} \rangle\|_2/\sqrt{\lambda}\notag\\
&\leq \biggl(\sum^d_{\ell = 1}\|\mu_{\ell, h}\|^2_1 \biggr)^{1/2}\cdot \|V_{k+1}^h\|_{\infty}/\sqrt{\lambda}\leq H\sqrt{d/\lambda},
\#
where the first inequality follows from the fact that $\Lambda^k_{1, h} \succeq \lambda I$, the second inequality follows from the H\"older's inequality, and the third inequality follows from Assumption \ref{asu::linMDP_frontdoor} and the fact that $V^h_{k+1} \leq H$. Finally, by plugging \eqref{eq::frontUCB_eq10} and \eqref{eq::frontUCB_eq11} into \eqref{eq::frontUCB_eq9}, we obtain for all $(s_h, m_h)\in\cS\times\cM$ that
\#\label{eq::frontdoor_UCB_8.5}
&|\psi_h(s_h, m_h)^\top \omega^k_{1, h} - (\PP_{h+1/2} V^k_{h+1})(s_h, m_h)| \notag\\
&\quad\leq \beta/\sqrt{2}\cdot \bigl(\psi_h(s_h, m_h)^\top (\Lambda^k_{1, h})^{-1} \psi_h(s_h, m_h)\bigr)^{1/2}\notag\\
&\quad\leq \beta\cdot\Bigl(\log \det\bigl(\Lambda^k_{1, h} + \psi_h(s_h, m_h)\psi_h(s_h, m_h)^\top\bigr) - \log\det(\Lambda^k_{1, h})\Bigr)^{1/2}\notag\\
&\quad=\Gamma^k_{h+1/2}(s_h, m_h),
\#
where we set $\beta = C''dH\sqrt{\log(d(T + nH)/\zeta)}$ for a sufficiently large absolute constant $C''>0$ and the last inequality follows from Lemma \ref{lem::Gamma}. Here $\Gamma^k_{h+1/2}$ is the UCB defined in \eqref{eq::UCB_frontdoor}. Recall that for all $(s_h, m_h)\in\cS\times\cM$, we define
\$
V^k_{h+1/2}(s_h, m_h) =  \min\bigl\{\psi_h(s_h, m_h)^\top \omega^k_{1, h} + \Gamma^k_{h+1/2}(s_h, m_h), H-h\bigr\}.
\$
Hence, by \eqref{eq::frontdoor_UCB_8.5}, for all $(s_h, m_h)\in\cS\times\cM$, it holds with probability at least $1 - 2\zeta$ that
\$
-\iota^k_{h+1/2}(s_h, m_h) &= V^k_{h+1/2}(s_h, m_h) - ( \PP_{h+1/2} V^k_{h+1})(s_h, m_h)\notag\\
&\leq \psi_h(s_h, m_h)^\top\omega_h^k + \Gamma^k_{h+1/2}(s_h, m_h) - (\PP_{h+1/2} V^k_{h+1})(s_h, m_h) \leq 2\Gamma^k_{h+1/2}(s_h, m_h),
\$
and
\$
\iota^k_{h+1/2}(s_h, m_h) &= -V^k_{h+1/2}(s_h, m_h) + (\PP_{h+1/2} V^k_{h+1})(s_h, m_h)\notag\\
&\leq \max\bigl\{(\PP_{h+1/2} V^k_{h+1})(s_h, m_h) - \psi_h(s_h, m_h)^\top\omega_{1, h}^k - \Gamma^k_{h+1/2}(s_h,m_h), \notag\\
&\qquad\qquad\quad (\PP_{h+1/2} V^k_{h+1})(s_h, m_h) - H + h\bigr\} \leq 0,
\$
where the second inequality follows from \eqref{eq::frontdoor_UCB_8.5} and the fact that $V^k_{h+1} \leq H-h-1$. In conclusion, it holds with probability at least $1 - 2\zeta$ that
\$
-2\Gamma^k_{h+1/2}(s_h, m_h) \leq \iota^k_{h+1/2}(s_h, m_h) \leq 0.
\$
Similarly, following from the proof of Lemma \ref{lem:backdoor_UCB} with Lemma \ref{lem::covering_Q} in place of Lemma \ref{lem::covering}, the reward $r_h$ in place of $R_h$, and the feature $\gamma_h$ in place of both $\psi_h$ and $\phi_h$, for all $(s_h, a_h)\in\cS\times\cA$, it holds with probability at least $1 - 2\zeta$ that
\$
-2\Gamma^k_{h}(s_h, a_h) \leq \iota^k_{h}(s_h, a_h) \leq 0.
\$
Thus, we complete the proof of Lemma \ref{lem::front_door}.
\end{proof}

\section{Auxiliary Lemma}
\label{sec::aux_lemma}
\begin{lemma}[Concentration of Self-Normalized Process \citep{abbasi2011improved, jin2019provably}]
\label{lem::concen_SNP}
Let $\{\epsilon_t\}_{t = 1}^\infty$ be a real-valued stochastic process adapted to the filtration $\{\cF_t\}^\infty_{t = 0}$. Let $\epsilon_t\given\cF_{t-1}$ be zero-mean and $\sigma$-sub-Gaussian. Let $\{\psi_t\}^\infty_{t = 0}$ be an $\RR^d$-valued stochastic process with $\psi_t \in \cF_{t-1}$. Let $\overline\Lambda_t = \overline\Lambda_0 + \sum^{t}_{\tau = 1}\psi_\tau \psi_\tau^\top$, where $\overline\Lambda_0$ is a positive definite matrix. Let $\delta >0$ be an absolute constant. It then holds with probability at least $1 - \delta$ that
\$
\biggl\|\sum^t_{\tau = 1}\psi_\tau\cdot \epsilon_\tau \biggr\|^2_{\overline\Lambda_t^{-1}} \leq 2\sigma^2\cdot\log \Bigl(\sqrt{\det(\overline\Lambda_t) / \det(\overline\Lambda_0) }\cdot\delta^{-1}\Bigr), \quad \forall t \geq 0.
\$
\end{lemma}
\begin{proof}
See \cite{abbasi2011improved} for a detailed proof.
\end{proof}

\begin{lemma}[Lemma D.4 of \cite{jin2019provably}]
\label{lem::Jin_D4}
Let $\{s_t\}_{t = 1}^\infty$ and $\{\psi_t\}_{t = 1}^\infty$ with $\|\psi_t\|_2 \leq 1$ be $\cS$-valued and $\RR^d$-valued stochastic processes adopted to the filtration $\{\cF_t\}^\infty_{t = 0}$, respectively. Let $\overline\Lambda_t = \overline\Lambda_0 + \sum^{t}_{\tau = 1}\psi_\tau \psi_\tau^\top$, where $\overline\Lambda_0 \succeq \lambda I$ is a positive definite matrix. Let $\sup_{s\in \cS}|V(s)| \leq H$ for all $V \in \cV$. Let $\delta >0$ be an absolute constant. It then holds with probability at least $1-\delta$ that
\$
&\biggl\|\sum^t_{\tau = 1}\psi_\tau\cdot \Bigl(V(s_\tau) - \EE\bigl[V(s_\tau) ~\big|~ \cF_{\tau - 1}\bigr]\Bigr) \biggr\|_{\overline\Lambda_t^{-1}} \notag\\
&\qquad\leq 4H^2\cdot \Bigl(d/2\cdot\log \bigl(\det(\overline\Lambda_t) / \det(\overline\Lambda_0)\bigr) + \log(\cN_\epsilon/\delta) \Bigr) + 8t^2\epsilon^2/\lambda.
\$
Here $\cN_\epsilon$ is the $\epsilon$-covering number of $\cV$ with respect to the metric $d(V, V') = \sup_{s\in\cS }|V(s ) - V'(s )|$ for all $V, V'\in\cV$.
\end{lemma}
\begin{proof}
The proof technique is similar to that of Lemma D.4 by \cite{jin2019provably}. For all $V \in \cV$, there exist an element $\tilde V$ in the $\epsilon$-covering of $\cV$ satisfying
\#\label{eq::JinD4_eq1}
d(V, \tilde V) = \sup_{s\in\cS }|V(s ) - \tilde V(s)| \leq \epsilon. 
\#
In the sequel, we define 
\#\label{eq::JinD4_eq2}
\Delta_V(\cdot) =   V(\cdot ) - \tilde V(\cdot ).
\# 
It then holds that
\#\label{eq::JinD4_eq3}
&\biggl\|\sum^t_{\tau = 1}\psi_\tau\cdot \Bigl(V(s_\tau) - \EE\bigl[V(s_\tau) ~\big|~ \cF_{\tau - 1}\bigr]\Bigr) \biggr\|_{\overline\Lambda_t^{-1}}\notag\\
&\quad\leq 2 \biggl\|\sum^t_{\tau = 1}\psi_\tau\cdot \Bigl(\tilde V(s_\tau) - \EE\bigl[\tilde V(s_\tau) ~\big|~ \cF_{\tau - 1}\bigr]\Bigr) \biggr\|_{\overline\Lambda_t^{-1}} \\
&\quad\quad+ 2 \biggl\|\sum^t_{\tau = 1}\psi_\tau\cdot \Bigl(\Delta_V(s_\tau) - \EE\bigl[\Delta_V(s_\tau) ~\big|~ \cF_{\tau - 1}\bigr]\Bigr) \biggr\|_{\overline\Lambda_t^{-1}}.\notag
\#
Note that $|\tilde V(s)| \leq H$ for all $s\in\cS$. Hence, following from Lemma \ref{lem::concen_SNP} and a union bound argument, it holds with probability at least $1 - \delta$ that
\#\label{eq::JinD4_eq4}
&2 \biggl\|\sum^t_{\tau = 1}\psi_\tau\cdot \Bigl(\tilde V(s_\tau) - \EE\bigl[\tilde V(s_\tau) ~\big|~ \cF_{\tau - 1}\bigr]\Bigr) \biggr\|_{\overline\Lambda_t^{-1}} \notag\\
&\qquad\leq 4H^2\cdot \Bigl(d/2\cdot\log \bigl(\det(\overline\Lambda_t) / \det(\overline\Lambda_0)\bigr) + \log(\cN_\epsilon/\delta) \Bigr),
\#
where $\cN_\epsilon$ is the $\epsilon$-covering number of $\cV$. Meanwhile, it follows from \eqref{eq::JinD4_eq1} and \eqref{eq::JinD4_eq2} that $|\Delta_V(s)| \leq \epsilon$ for all $s\in\cS$. Hence, we have
\#\label{eq::JinD4_eq5}
2 \biggl\|\sum^t_{\tau = 1}\psi_\tau\cdot \Bigl(\Delta_V(s_\tau) - \EE\bigl[\Delta_V(s_\tau) ~\big|~ \cF_{\tau - 1}\bigr]\Bigr) \biggr\|_{\overline\Lambda_t^{-1}}\leq 8t^2\epsilon^2/\lambda,
\#
where the inequality follows from the fact that $\overline\Lambda_t \succeq \lambda  I$. By plugging \eqref{eq::JinD4_eq4} and \eqref{eq::JinD4_eq5} into \eqref{eq::JinD4_eq3}, it holds with probability at least $1 - \delta$ that
\$
&\biggl\|\sum^t_{\tau = 1}\psi_\tau\cdot \Bigl(V(s_\tau) - \EE\bigl[V(s_\tau) ~\big|~ \cF_{\tau - 1}\bigr]\Bigr) \biggr\|_{\overline\Lambda_t^{-1}} \notag\\
&\qquad\leq 4H^2\cdot \Bigl(d/2\cdot\log \bigl(\det(\overline\Lambda_t) / \det(\overline\Lambda_0)\bigr) + \log(\cN_\epsilon/\delta) \Bigr) + 8t^2\epsilon^2/\lambda,
\$
which concludes the proof of Lemma \ref{lem::Jin_D4}.
\end{proof}

\begin{lemma}[Upper Bound of Parameter \citep{jin2019provably}]
\label{lem::param}
Under Assumption \ref{asu::linMDP_backdoor}, It holds that
\#
\|\omega^k_h\|_2 \leq H\bigl(d(k+n)/\lambda\bigr)^{1/2}, \quad \forall (k, h)\in[K]\times[H].
\#
\end{lemma}
\begin{proof}
See \cite{jin2019provably} for a detailed proof.
\end{proof}

\begin{lemma}[Covering Number of $\cV$ \citep{jin2019provably}]
\label{lem::covering}
Let $\cV$ be a class of functions $V$ satisfying
\#\label{eq::covering_V}
V(\cdot) = \min\bigl\{\max_{a\in\cA}\psi(\cdot,a)^\top\omega + \Gamma(\cdot, a), H-h\bigr\},
\#
where 
\#\label{eq::covering_gamma}
\Gamma(\cdot, \cdot) = \sqrt{2}\beta\cdot\Bigl(\log \det\bigl(\Lambda + \psi(\cdot, \cdot)\psi(\cdot, \cdot)^\top\bigr) - \log\det(\Lambda)\Bigr)^{1/2}.
\#
Here the function $V$ is parameterized by $(\omega, \Lambda)$ and the parameter $\beta$ is fixed. Let $\psi(\cdot, \cdot)$ be an $\RR^d$-valued function and $\Lambda\in\RR^{d\times d}$. Let $\|\psi(s, a)\|_2 \leq 1$ for all $(s, a)\in\cS\times\cA$. For $\|\omega\|_2 \leq L$, $\Lambda \succeq \lambda I$, $\beta\in[0, B]$, and $\epsilon >0$, there exist an $\epsilon$-covering of $\cV$ with respect to the metric $d(V, V') = \sup_{s \in\cS}|V(s ) - V'(s )|$, such that the covering number $\cN_\epsilon$ is upper bounded as follows,
\$
\log \cN_\epsilon \leq  d\cdot \log(1 + 4L/\epsilon) + d^2 \cdot \log\bigl(1 + 16B^2d^{1/2}/(\epsilon^2\lambda)\bigr).
\$
\end{lemma}
\begin{proof}
The proof technique is similar to that of Lemma D.6 by \cite{jin2019provably}. Let $V_1$ and $V_2$ be the functions defined in \eqref{eq::covering_V}, which are parameterized by $(\omega_1, \Lambda_1)$ and $(\omega_2, \Lambda_2)$, respectively. Note that
\#\label{eq::covering_1}
d(V_1, V_2) &\leq \sup_{  s \in\cS } \bigl| \min\bigl\{\max_{a\in\cA}\psi(s, a)^\top\omega_1 + \Gamma_1(s, a), H-h\bigr\}  \notag\\
&\qquad\qquad\qquad-  \min\bigl\{\max_{a\in\cA}\psi(s, a)^\top\omega_2 + \Gamma_2(s, a), H-h\bigr\}\bigr|\notag\\
&\leq \sup_{(s, a)\in\cS\times\cA}|\psi(s, a)^\top(\omega_1-\omega_2) + \Gamma_1(s, a) - \Gamma_2(s, a)|,
\#
where the second inequality follows from the fact that $\min\{\cdot, H-h\} $ and $\max_{a\in\cA}$ are contraction mappings. Here we define $\Gamma_1$ and $\Gamma_2$ in \eqref{eq::covering_gamma} with $\Lambda=\Lambda_1$ and $\Lambda=\Lambda_2$, respectively. Meanwhile, following from the matrix determinant lemma, we have
\$
\Gamma_1(s, a) &= \sqrt{2}\beta\cdot\Bigl(\log \det\bigl(\Lambda_1 + \psi(s, a)\psi(s, a)^\top\bigr) - \log\det(\Lambda_1)\Bigr)^{1/2}\notag\\
& = \sqrt{2}\beta\cdot\Bigl(\log\bigl(1 + \psi(s, a)^\top\Lambda_1^{-1}\psi(s,a)\bigr)\Bigr)^{1/2}, \quad \forall (s, a)\in\cS\times\cA.
\$
Thus, following from the inequalities $|\sqrt{x} - \sqrt{y}|\leq \sqrt{|x-y|}$ and $|\log(1+x) - \log(1+y)| \leq |x - y|$ for all $x, y\geq0$, we have
\#\label{eq::covering_2}
|\Gamma_1(s, a) - \Gamma_2(s, a)| &\leq \sqrt{2}\beta\cdot\Bigl(\bigl|\log\bigl(1 + \psi(s, a)^\top\Lambda_1^{-1}\psi(s,a)\bigr) - \log\bigl(1 + \psi(s, a)^\top\Lambda_2^{-1}\psi(s,a)\bigr)\bigr|\Bigr)^{1/2}\notag\\
&\leq\sqrt{2}\beta\cdot\Bigl(|\psi(s, a)^\top(\Lambda_1^{-1} - \Lambda_2^{-1})\psi(s,a)|\Bigr)^{1/2}.
\#
Combining \eqref{eq::covering_1} and \eqref{eq::covering_2}, we have
\#\label{eq::covering_3}
d(V_1, V_2)&\leq\sup_{(s, a)\in\cS\times\cA}|\psi(s, a)^\top(\omega_1-\omega_2) + \Gamma_1(s, a) - \Gamma_2(s, a)| \notag\\
&\leq \sup_{\|\psi\|_2\leq 1} |\psi^\top(\omega_1-\omega_2)| + \sqrt{2}\beta\cdot\sup_{\|\psi\|_2\leq 1}\bigl(|\psi^\top(\Lambda_1^{-1} - \Lambda_2^{-1})\psi|\bigr)^{1/2}\notag\\
& =\|\omega_1 - \omega_2\|_2 + \|2\beta^2\cdot\Lambda_1^{-1} - 2\beta^2\cdot\Lambda_2^{-1}\|^{1/2}_{\textrm{OP}}\notag\\
&\leq \|\omega_1 - \omega_2\|_2 + \|2\beta^2\cdot\Lambda_1^{-1} - 2\beta^2\cdot\Lambda_2^{-1}\|^{1/2}_{\textrm{F}},
\#
where we denote by $\|\cdot\|_{\textrm{OP}}$ and $\|\cdot\|_{\textrm{F}}$ the operator norm and Frobenius norm, respectively. For $\Lambda \succeq \lambda I$ and $\beta\in[0, B]$, it holds that $\|2\beta^2\cdot\Lambda^{-1}\|_{\textrm{F}} \leq 2B^2d^{1/2}\lambda^{-1}$. Meanwhile, let $\cN_{\omega, \epsilon}$ be the $\epsilon/2$-covering number of $\{\omega\in\RR^d: \|\omega\|_2\leq L\}$, and $\cN_{A, \epsilon}$ be the $\epsilon^2/4$-covering number of $\{A\in\RR^{d\times d}: \|A\|_{\textrm{F}}\leq 2B^2d^{1/2}\lambda^{-1}\}$. It is known that \citep{vershynin2010introduction}
\$
\cN_{\omega, \epsilon} \leq (1 + 4L/\epsilon)^{d}, \qquad \cN_{A, \epsilon} \leq \bigl(1 + 16B^2d^{1/2}/(\lambda \epsilon^2)\bigr)^{d^2}.
\$
Hence, by \eqref{eq::covering_3}, we obtain that
\$
\log \cN_\epsilon &\leq \log(\cN_{\omega, \epsilon}\cdot \cN_{A, \epsilon})\leq d\cdot \log(1 + 4L/\epsilon) + d^2 \cdot \log\bigl(1 + 16B^2d^{1/2}/(\epsilon^2\lambda)\bigr),
\$
which concludes the proof of Lemma \ref{lem::covering}.
\end{proof}

\begin{lemma}[Covering Number of $Q$ \citep{jin2019provably}]
\label{lem::covering_Q}
Let $\cQ$ be a class of functions $Q$ satisfying
\#\label{eq::covering_Q}
Q(\cdot, \cdot) = \min\bigl\{\psi(\cdot,\cdot)^\top\omega + \Gamma(\cdot, \cdot), H-h\bigr\},
\#
where 
\#\label{eq::covering_gamma_Q}
\Gamma(\cdot, \cdot) = \sqrt{2}\beta\cdot\Bigl(\log \det\bigl(\Lambda + \psi(\cdot, \cdot)\psi(\cdot, \cdot)^\top\bigr) - \log\det(\Lambda)\Bigr)^{1/2}.
\#
Here the function $Q$ is parameterized by $(\omega, \Lambda)$ and the parameter $\beta$ is fixed. Let $\psi(\cdot, \cdot)$ be an $\RR^d$-valued function and $\Lambda\in\RR^{d\times d}$. Let $\|\psi(s, m)\|_2 \leq 1$ for all $(s, m)\in\cS\times\cM$. For $\|\omega\|_2 \leq L$, $\Lambda \succeq \lambda I$, $\beta\in[0, B]$, and $\epsilon >0$, there exist an $\epsilon$-covering of $\cQ$ with respect to the metric $d(V, V') = \sup_{(s, m) \in\cS\times\cM}|Q(s, m ) - Q'(s, m )|$, such that the covering number $\cN_\epsilon$ is upper bounded as follows,
\$
\log \cN_\epsilon \leq  d\cdot \log(1 + 4L/\epsilon) + d^2 \cdot \log\bigl(1 + 16B^2d^{1/2}/(\epsilon^2\lambda)\bigr).
\$
\end{lemma}
\begin{proof}
The proof is similar to that of Lemma \ref{lem::covering}. Let $Q_1$ and $Q_2$ be the functions defined in \eqref{eq::covering_Q}, which are parameterized by $(\omega_1, \Lambda_1)$ and $(\omega_2, \Lambda_2)$, respectively. Note that
\#\label{eq::covering_1}
d(Q_1, Q_2) &\leq \sup_{ \min\bigl\{ (s, m) \in\cS\times\cM } \bigl| \psi(s, m)^\top\omega_1 + \Gamma_1(s, m), H-h\bigr\}  \notag\\
&\qquad\qquad\qquad-  \min\bigl\{\psi(s, m)^\top\omega_2 + \Gamma_2(s,m), H-h\bigr\}\bigr|\notag\\
&\leq \sup_{(s, m)\in\cS\times\cM}|\psi(s, m)^\top(\omega_1-\omega_2) + \Gamma_1(s, m) - \Gamma_2(s, m)|,
\#
where the second inequality follows from the fact that $\min\{\cdot, H-h\} $ is a contraction mapping. Here we define $\Gamma_1$ and $\Gamma_2$ in \eqref{eq::covering_gamma_Q} with $\Lambda=\Lambda_1$ and $\Lambda=\Lambda_2$, respectively. The rest of the proof is the same as that of Lemma \ref{lem::covering}. We omit the proof and refer to the proof of Lemma \ref{lem::covering} for the details.
\end{proof}

\begin{lemma}[Concentration of Self-Normalized Process]
\label{lem::self_norm_process}
Let $\lambda = 1$ and $\beta = CdH\sqrt{\log(d(T + nH)/\zeta)}$. Let $\zeta >0$ be an absolute constant. It holds with probability at least $1 - 2\zeta$ that
\$
 \biggl\|\sum^4_{\ell = 1}S_{\ell, h}\biggr\|_{(\Lambda^k_h)^{-1}}\leq C'dH\sqrt{\log\bigl(2(C+1)d(T+nH)/\zeta\bigr)}, \quad \forall (k, h)\in[K]\times[H].
\$
where $C$ and $C'$ are positive absolute constants and $C'$ is independent of $C$.
\end{lemma}
\begin{proof}
Recall that we define
\$
S_{1, h} &=  \sum^{k-1}_{\tau = 1} \psi_h(s^\tau_h, a^\tau_h) \cdot \bigl(V^k_{h+1}(s^\tau_{h+1}) - (\PP_h V^k_{h+1})(s^\tau_h, a^\tau_h)  \bigr),\\
S_{2, h} &= \sum^n_{i = 1} \phi_h(s_h^i, a_h^i, u_h^i) \cdot \bigl(V^k_{h+1}(s^i_{h+1})-(\tilde\PP_h V^k_{h+1})(s_h^i, a_h^i, u_h^i) \bigr),\notag\\
S_{3, h} &=  \sum^{k-1}_{\tau = 1} \psi_h(s^\tau_h, a^\tau_h) \cdot \bigl( r^\tau_h -  R(s^\tau_h, a^\tau_h) \bigr),\quad S_{4, h} =\sum^n_{i = 1} \phi_h(s_h^i, a_h^i, u_h^i) \cdot \bigl( r^i_h -  \EE[r_h\given s^i_h, a^i_h, u^i_h] \bigr).\notag
\$
We define $\cF_{-n + i}$ the $\sigma$-algebra generated by the set $\{(s^\ell_h, a^\ell_h, u^\ell_h, r^\ell_h)\}_{(\ell, h)\in[i]\times[H]}$ with timestep index $-n+i$. The set of $\sigma$-algebra $\{\cF_{-n + i}\}_{i\in[n]}$ captures the data generation process in the offline setting. We attach $\{\cF_{-n + i}\}_{i\in[n]}$ to the $\sigma$-algebra $\{\cF_{k, h, m}\}_{(k, h, m)\in[K, H, 2]}$ with timestep index $t$ defined in Definition \ref{def::filtration_backdoor} to obtain the complete filtration. By Lemma \ref{lem::concen_SNP} with such a complete filtration, it holds with probability at least $1 - \zeta$ that
\#\label{eq::self_norm_eq1}
&\|S_{1, h} + S_{2, h}\|_{(\Lambda^k_h)^{-1} }\notag\\
&\qquad\leq 4H^2\cdot \Bigl(d/2\cdot\log \bigl(\det(\Lambda^k_h) / \det(\Lambda_0)\bigr) + \log(2\cN_\epsilon/\zeta) \Bigr) + 8(n+k)^2\epsilon^2/\lambda,
\#
where $\Lambda_0 = \lambda I$ and 
\$
&\Lambda^k_h = \sum^{k-1}_{\tau = 1} \psi_h(s^\tau_h, a^\tau_h) \psi_h(s^\tau_h, a^\tau_h)^\top+ \sum^n_{i = 1} \phi_h(s_h^i, a_h^i, u_h^i)\phi_h(s_h^i, a_h^i, u_h^i)^\top + \lambda   I.
\$
Similarly, by Lemma \ref{lem::concen_SNP}, it holds with probability at least $1 - \zeta$ that 
\#\label{eq::self_norm_eq2}
&\|S_{3, h} + S_{4, h}\|_{(\Lambda^k_h)^{-1} }  \leq 4H^2\cdot \Bigl(d/2\cdot\log \bigl(\det(\Lambda^k_h) / \det(\Lambda_0)\bigr) \Bigr).
\#
Note that
\$
\Lambda^k_h &=\sum^{k-1}_{\tau = 1} \psi_h(s^\tau_h, a^\tau_h) \psi_h(s^\tau_h, a^\tau_h)^\top+\sum^n_{i = 1} \phi_h(s_h^i, a_h^i, u_h^i)\phi_h(s_h^i, a_h^i, u_h^i)^\top + \lambda I \notag\\
&\preceq (k + n + \lambda)  I.
\$
Meanwhile, recall that $\Lambda_0 = \lambda I$. Thus, we obtain that
\#\label{eq::self_norm_eq3}
\det(\Lambda^k_h) / \det(\Lambda_0) \leq (k + n + \lambda)/\lambda.
\#
On the other hand, we obtain from Lemma \ref{lem::param} and Lemma \ref{lem::covering} that
\#\label{eq::self_norm_eq4}
\log \cN_\epsilon \leq d\cdot\bigl(1 + 4H\sqrt{d(n+k)}/(\epsilon\sqrt{\lambda})\bigr) + d^2\cdot \log\bigl(1 + 16 \beta^2\sqrt{d}/(\epsilon^2\lambda)\bigr),
\#
where we set $\beta = CdH\sqrt{\log(d(T+nH)/\zeta)}$. Finally, by setting $\epsilon = dH/(n+k)$ in \eqref{eq::self_norm_eq1}, plugging \eqref{eq::self_norm_eq3} and \eqref{eq::self_norm_eq4} into \eqref{eq::self_norm_eq1} and \eqref{eq::self_norm_eq2}, respectively, and setting $\lambda = 1$, we obtain that
\$
\biggl\|\sum^4_{\ell = 1}S_{\ell, h}\biggr\|_{(\Lambda^k_h)^{-1}}&\leq\|S_{1, h} + S_{2, h}\|_{(\Lambda^k_h)^{-1} } + \|S_{3, h} + S_{4, h}\|_{(\Lambda^k_h)^{-1} }\notag\\
&\leq C'dH\sqrt{\log\bigl(2(C+1)d(T + nH)/\zeta\bigr)},
\$
which holds with probability at least $1 - 2\zeta$. Here $T = HK$ and $C$, $C'$ are absolute constants, where $C'$ is independent of $C$. Thus, we complete the proof of Lemma \ref{lem::self_norm_process}.
\end{proof}

\begin{lemma}
\label{lem::Gamma}
Let $\Lambda_t \in \RR^{d\times d}$ be a positive definite matrix satisfying $\Lambda_t \succeq I$. Let $\psi_t(\cdot, \cdot)$ be a $\RR^d$-valued function such that $\|\psi_t(\cdot, \cdot)\|_2 \leq 1$. Let $\Lambda_{t+1}(\cdot, \cdot) = \Lambda_t + \psi_t(\cdot, \cdot)\psi_t(\cdot, \cdot)^\top$. It then holds that
\$
\psi_t(\cdot,\cdot)^\top (\Lambda_t)^{-1} \psi_t(\cdot, \cdot) \leq 2\log \det\bigl(\Lambda_{t+1}(\cdot, \cdot)\bigr) - 2\log\det(\Lambda_t).
\$
\end{lemma}
\begin{proof}
Note that $\Lambda_t \succeq I$. Thus, it holds that
\$
0 \leq \psi_t(\cdot,\cdot)^\top (\Lambda_t)^{-1} \psi_t(\cdot, \cdot) \leq \|\psi_t(\cdot, \cdot)\|^2_2 \leq 1.
\$
It then follows from the inequality $x\leq 2\log(1 + x)$ for all $x\in[0, 1]$ that
\#\label{eq::gamma_eq1}
\psi_t(\cdot,\cdot)^\top (\Lambda_t)^{-1} \psi_t(\cdot, \cdot) \leq 2\log \bigl( 1 + \psi_t(\cdot,\cdot)^\top (\Lambda_t)^{-1} \psi_t(\cdot, \cdot)\bigr).
\#
Meanwhile, it follows from the matrix determinant lemma that
\#\label{eq::gamma_eq2}
\det\bigl(\Lambda_{t+1}(\cdot, \cdot)\bigr) = \det(\Lambda_t)\cdot\bigl( 1 + \psi_t(\cdot,\cdot)^\top (\Lambda_t)^{-1} \psi_t(\cdot, \cdot)\bigr).
\#
Finally, combining \eqref{eq::gamma_eq1} and \eqref{eq::gamma_eq2}, we conclude that
\$
\psi_t(\cdot,\cdot)^\top (\Lambda_t)^{-1} \psi_t(\cdot, \cdot) \leq 2\log \det\bigl(\Lambda_{t+1}(\cdot, \cdot)\bigr) - 2\log\det(\Lambda_t),
\$
which concludes the proof of Lemma \ref{lem::Gamma}.
\end{proof}

\end{document}

%% file: abs.tex

\begin{abstract}
Empowered by expressive function approximators such as neural networks, deep reinforcement learning (DRL) achieves tremendous empirical successes. However, learning expressive function approximators requires collecting a large dataset (interventional data) by interacting with the environment. Such a lack of sample efficiency prohibits the application of DRL to critical scenarios, e.g., autonomous driving and personalized medicine, since trial and error in the online setting is often unsafe and even unethical. In this paper, we study how to incorporate the dataset (observational data) collected offline, which is often abundantly available in practice, to improve the sample efficiency in the online setting.

To incorporate the observational data, we face two challenges. (a) The behavior policy that generates the observational data may depend on possibly unobserved random variables (confounders), which at the same time, affect the received rewards and transition dynamics. Such a confounding issue makes the observational data uninformative and even misleading for decision making in the online setting. (b) Exploration in the online setting requires quantifying the uncertainty that remains given both the observational and interventional data. In particular, it remains unclear how to quantify the amount of information carried over by the confounded observational data, which plays a key role in constructing the bonus and characterizing the regret. 

To address the two challenges, we propose the deconfounded optimistic value iteration (DOVI) algorithm, which incorporates the confounded observational data in a provably efficient manner. More specifically, DOVI explicitly adjusts for the confounding bias in the observational data, where the confounders are partially observed or unobserved. In both cases, such adjustments allow us to construct the bonus based on a notion of information gain, which takes into account the amount of information acquired from the offline setting. In particular, we prove that the regret of DOVI is smaller than the optimal regret achievable in the pure online setting by a multiplicative factor, which decreases towards zero when the confounded observational data are more informative upon the adjustments. Our algorithm and analysis serve as a step towards causal reinforcement learning.
\end{abstract}

%% file: intro.tex

\section{Introduction}\label{sec::intro}
In reinforcement learning (RL) \citep{sutton2018reinforcement}, an agent maximizes its expected total reward by sequentially interacting with the environment. Empowered by the breakthrough in neural networks, which serve as expressive function approximators, deep reinforcement learning (DRL) achieves significant empirical successes in various scenarios, e.g., game playing \citep{silver2016mastering, silver2017mastering}, robotics \citep{kober2013reinforcement}, and natural language processing \citep{li2016deep}. Learning an expressive function approximator necessitates collecting a large dataset. Specifically, in the online setting, it requires the agent to interact with the environment for a large number of steps. For example, to learn a human-level policy for playing Atari games, the agent has to interact with a simulator for more than $10^8$ steps \citep{hessel2018rainbow}. However, in most scenarios, we do not have access to a simulator that allows for trial and error without any cost. Meanwhile, in critical scenarios, e.g., autonomous driving and personalized medicine, trial and error in the real world is unsafe and even unethical. As a result, it remains challenging to apply DRL to more scenarios.

To bypass such a barrier, we study how to incorporate the dataset collected offline, namely the observational data, to improve the sample efficiency of RL in the online setting \citep{levine2020offline}. In contrast to the interventional data, which are collected online in possibly expensive ways, the observational data are often abundantly available in various scenarios. For example, in autonomous driving, we have access to a large number of trajectories generated by the drivers. As another example, in personalized medicine, we have access to a large number of electronic health records generated by the doctors. However, to incorporate the observational data in a provably efficient way, we have to address two challenges. 

\begin{itemize}[leftmargin=*]
    \item The observational data are possibly confounded. Specifically, there often exist  unobserved random variables, namely confounders, that causally affect the agent and the environment at the same time. In particular, the policy used to generate the observational data, namely the behavior policy, possibly depends on the confounders. Meanwhile, the confounders possibly affect the received rewards and the transition dynamics. 
    
    In the example of autonomous driving \citep{de2019causal, li2020make}, the driver may react, e.g., by pulling the break (policy), based on a traffic situation, e.g., icy roads (confounder), that is not captured by the sensor. Meanwhile, icy roads may lead to car accidents (reward/transition). Also, in the example of personalized medicine \citep{murphy2003optimal, chakraborty2014dynamic}, the doctor may treat the patient, e.g., by prescribing a medicine (policy), based on a clinical finding, e.g., appetite loss (confounder), that is not reflected in the record. Meanwhile, appetite loss may lead to weight loss (reward/transition). 
    
    Such a confounding issue makes the observational data uninformative and even misleading for identifying and estimating the causal effect of a policy, which is crucial for decision making in the online setting. In the example of autonomous driving, it is unclear from the observational data whether pulling the break causes car accidents. Also, in the example of personalized medicine, it is unclear from the observational data whether taking the medicine causes weight loss. 
    
    \item Even without the confounding issue, it remains unclear how the observational data may facilitate exploration in the online setting, which is the key to the sample efficiency of RL. At the core of exploration is uncertainty quantification. Specifically, quantifying the uncertainty that remains given the dataset collected up to the current step, including the observational data and the interventional data, allows us to construct a bonus. When incorporated into the reward, such a bonus encourages the agent to explore the less visited state-action pairs that have more uncertainty. In particular, constructing such a bonus requires quantifying the amount of information carried over by the observational data from the offline setting, which also plays a key role in characterizing the regret, especially how much the observational data may facilitate reducing the regret.
    
    Uncertainty quantification becomes even more challenging when the observational data are confounded. Specifically, as the behavior policy depends on the confounders, which are unobserved, there is a mismatch between the data generating processes in the offline setting and the online setting. As a result, it remains challenging to quantify how much information carried over from the offline setting is useful for the online setting, as the observational data are uninformative and even misleading due to the confounding issue.   
    
\end{itemize}

\vskip4pt
\noindent{\bf Contribution.} To study causal reinforcement learning, we propose a class of Markov decision processes (MDPs), namely confounded MDPs, which captures the data generating processes in both the offline setting and the online setting as well as their mismatch due to the confounding issue. In particular, we study two tractable cases of confounded MDPs in the episodic setting with linear function approximation \citep{yang2019sample, yang2019reinforcement, jin2019provably, cai2019provably}.  
\begin{itemize}[leftmargin=*]
\item In the first case, the confounders are partially observed in the observational data. Assuming that an observed subset of the confounders satisfies the backdoor criterion \citep{pearl2009causality}, we propose the deconfounded optimistic value iteration (DOVI) algorithm. Specifically, DOVI explicitly corrects for the confounding bias in the observational data using the backdoor adjustment.   

\item In the second case, the confounders are unobserved in the observational data. Assuming that there exists an observed set of intermediate states that satisfies the frontdoor criterion \citep{pearl2009causality}, we propose an extension of DOVI, namely DOVI$^+$, which explicitly corrects for the confounding bias in the observational data using the composition of two backdoor adjustments.

\end{itemize}
In both cases, the adjustments allow DOVI and DOVI$^+$ to incorporate the observational data into the interventional data while bypassing the confounding issue. It further enables estimating the causal effect of a policy on the received rewards and the transition dynamics with an enlarged effective sample size. Moreover, such adjustments allow us to construct the bonus based on a notion of information gain, which takes into account the amount of information carried over from the offline setting.

In particular, we prove that DOVI and DOVI$^+$ attain the $\Delta_H\cdot \sqrt{d^3 H^3 T}$-regret up to logarithmic factors, where $d$ is the dimension of features, $H$ is the length of each episode, and $T = HK$ is the number of steps taken in the online setting, where $K$ is the number of episodes. Here the multiplicative factor $\Delta_H > 0$ depends on $d$, $H$, and a notion of information gain that quantifies the amount of information obtained from the interventional data additionally when given the properly adjusted observational data. When the observational data are unavailable or uninformative upon the adjustments, $\Delta_H$ is a logarithmic factor. Correspondingly, DOVI and DOVI$^+$ attain the optimal $\sqrt{T}$-regret achievable in the pure online setting \citep{yang2019sample, yang2019reinforcement, jin2019provably, cai2019provably}. When the observational data are sufficiently informative upon the adjustments, $\Delta_H$ decreases towards zero as the effective sample size of the observational data increases, which quantifies how much the observational data may facilitate exploration in the online setting.

\vskip4pt
\noindent{\bf Related Work.}
Our work is based on the study of RL in the pure online setting, which focuses on attaining the optimal regret. See, e.g., \cite{auer2007logarithmic, jaksch2010near, osband2014generalization, azar2017minimax, yang2019sample, yang2019reinforcement, jin2019provably} and the references therein. In contrast, we study a class of confounded MDPs, which captures a combination of the online setting and the offline setting.

Our work is related to the study of causal bandit \citep{lattimore2016causal}. The goal of causal bandit is to obtain the optimal intervention in the online setting where the data generating process is described by a causal diagram. \cite{lu2019regret} propose the causal upper confidence bound (C-UCB) and causal Thompson Sampling (C-TS) algorithms, which attain the $\sqrt{T}$-regret. \cite{sen2017identifying} propose an algorithm based on importance sampling in policy evaluation. In the pure offline setting, \cite{kallus2018confounding,kallus2018policy} propose algorithms for contextual bandit with confounders in the observational data. Their algorithms are based on the analysis of sensitivity \citep{manski1990nonparametric, tan2006distributional, balke2013counterfactuals, zhang2017transfer}, which characterizes the worst-case difference between the causal effect and the conditional density obtained from the confounded observational data. In a combination of the online setting and the offline setting, \cite{forney2017counterfactual} study multi-armed bandit with both the interventional data and the confounded observational data. In contrast to this line of work, we study causal RL in a combination of the online setting and the offline setting. Causal RL is more challenging than causal bandit, which corresponds to $H = 1$, as it involves the transition dynamics, which makes exploration more difficult.

Our work is related to the study of causal RL considered in various settings.  \cite{zhang2019near} propose a model-based RL algorithm that solves dynamic treatment regimes (DTR), which involve a combination of the online setting and the offline setting. Their algorithm hinges on the analysis  of sensitivity \citep{manski1990nonparametric, tan2006distributional, balke2013counterfactuals, zhang2017transfer}, which constructs a set of feasible models of the transition dynamics based on the confounded observational data. Correspondingly, their algorithm achieves exploration by choosing an optimistic model of the transition dynamics from such a feasible set. In contrast, we propose a model-free RL algorithm, which achieves exploration through the bonus based on a notion of information gain. It is worth mentioning that the assumption of \cite{zhang2019near} is weaker than ours as theirs does not allow for identifying the causal effect. As a result of partial identification, the regret of their algorithm is the same as the regret in the pure online setting as $T\to+\infty$. In contrast, the regret of our algorithm is smaller than the regret in the pure online setting by a multiplicative factor for all $T$. \cite{lu2018deconfounding} propose a model-based RL algorithm in a combination of the online setting and the offline setting. Their algorithm uses a variational autoencoder (VAE) for estimating a structural causal model (SCM) based on the confounded observational data. In particular, their algorithm utilizes the actor-critic algorithm to obtain the optimal policy in such an SCM. However, the regret of their algorithm remains unclear. \cite{buesing2018woulda} propose a model-based RL algorithm in the pure online setting that learns the optimal policy in a partially observable Markov decision process (POMDP). The regret of their algorithm also remains unclear.
